%% file: neurips_2026.tex
\documentclass{article}

\PassOptionsToPackage{numbers}{natbib}

\usepackage[main,final]{neurips_2026}

\usepackage[utf8]{inputenc} 
\usepackage[T1]{fontenc}    
\usepackage{url}            
\usepackage{booktabs}       
\usepackage{amsfonts}       
\usepackage{nicefrac}       
\usepackage{microtype}      
\usepackage{xcolor}         

\usepackage{graphicx}
\usepackage{enumitem}
\usepackage{threeparttable}
\usepackage{multirow}
\usepackage{amsmath,amssymb,amsthm}
\usepackage{algorithm}
\usepackage{algpseudocode}

\theoremstyle{plain}
\newtheorem{theorem}{Theorem}
\newtheorem{proposition}{Proposition}

\theoremstyle{definition}

\theoremstyle{remark}

\usepackage{multirow}
\usepackage{makecell}

\usepackage[utf8]{inputenc}
\usepackage[T1]{fontenc}
\usepackage{times}
\usepackage{microtype}
\usepackage{amsmath,amssymb,amsfonts}
\usepackage{booktabs}
\usepackage{colortbl}
\usepackage[table]{xcolor}
\usepackage{tcolorbox}
\usepackage{pifont}
\usepackage{geometry}
\usepackage{array}
\usepackage{caption}
\usepackage{listings}
\usepackage{fancyhdr}

\tcbuselibrary{skins,breakable,listings}

\newcommand{\cmark}{\ding{51}}
\newcommand{\xmark}{\ding{55}}

\lstdefinestyle{pythonstyle}{
  language=Python,
  basicstyle=\ttfamily\scriptsize,
  keywordstyle=\color{blue!70!black}\bfseries,
  stringstyle=\color{red!60!black},
  commentstyle=\color{green!50!black}\itshape,
  breaklines=true,
  frame=none,
  xleftmargin=2pt,
  aboveskip=2pt,
  belowskip=2pt,
}

\usepackage[colorlinks=true,
            citecolor=blue,
            linkcolor=red,
            anchorcolor=green]{hyperref}

\definecolor{sftgray}{RGB}{235,235,235}
\definecolor{simctblue}{RGB}{232,242,255}
\newcommand{\sftc}[1]{\cellcolor{sftgray}{#1}}
\newcommand{\simctc}[1]{\cellcolor{simctblue}{#1}}

\usepackage{array}
\newcolumntype{L}[1]{>{\raggedright\arraybackslash}p{#1}}

\title{SimCT: Recovering Lost Supervision for Cross-Tokenizer On-Policy Distillation}

\author{%
  \textbf{Jie Sun}$^{1,2,*,\ddagger}$ \quad
  \textbf{Mao Zheng}$^{2,*}$ \quad
  \textbf{Mingyang Song}$^{2,*}$ \quad
  \textbf{Qiyong Zhong}$^{2}$ \\
  \textbf{Yilin Cheng}$^{3}$ \quad
  \textbf{Bichuan Feng}$^{3}$ \quad
  \textbf{Pengfei Liu}$^{1}$ \quad
  \textbf{Junfeng Fang}$^{4,\dagger}$ \quad
  \textbf{Xiang Wang}$^{4,\dagger}$ \\[4pt]
  {\small $^{1}$\,Shanghai Innovation Institute \quad
  $^{2}$\, Foundation Model Department, Tencent} \\
  {\small $^{3}$\,Zhongguancun Academy \quad
  $^{4}$\,National University of Singapore} \\[4pt]
  \texttt{sunjie2019@mail.ustc.edu.cn}
}


\fancypagestyle{firstpagestyle}{%
  \fancyhf{}

  \fancyhead[L]{%
    \raisebox{-14pt}[0pt][0pt]{%
      \raisebox{-0.8\height}{%
        \includegraphics[height=0.48cm]{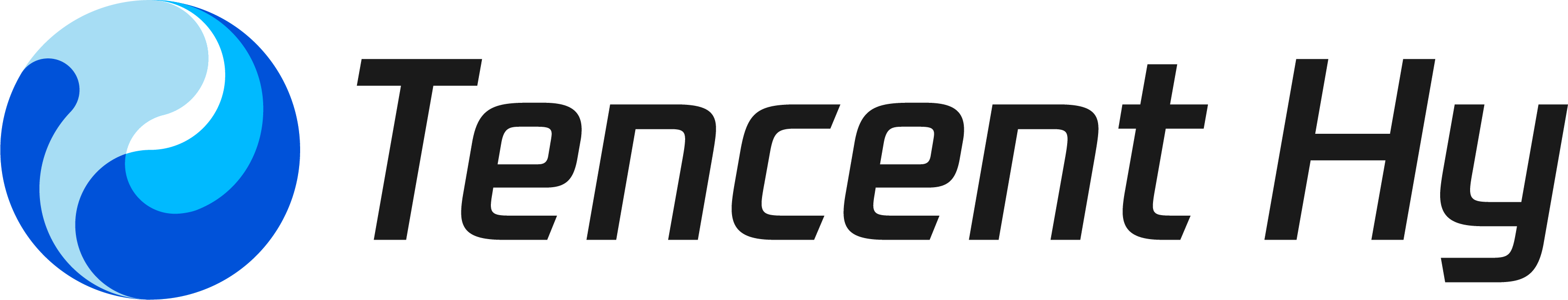}%
      }%
    }%
  }

  \fancyhead[R]{%
    \raisebox{-14pt}[0pt][0pt]{%
      \raisebox{-1\height}{%
        \textcolor{gray}{%
          \fontsize{9pt}{11pt}\selectfont October 15, 2026%
        }%
      }%
    }%
  }

}

\begin{document}
\setlength{\headheight}{24pt}
\addtolength{\topmargin}{-12pt}

\maketitle
\thispagestyle{firstpagestyle}

\begingroup
\renewcommand{\thefootnote}{}
\begin{NoHyper}
\footnotetext{$^{*}$\,Equal Contribution.\quad $^{\dagger}$\,Corresponding Authors.\quad $^{\ddagger}$\,Work done during internship at Tencent.}
\end{NoHyper}
\endgroup

\input{chapters/0-abs}

\input{chapters/1-intro}

\input{chapters/2-related_work}

\input{chapters/3-method}

\input{chapters/4-experiment}

\input{chapters/6-limitation}

\input{chapters/5-conclusion}

\newpage

\begin{ack}
This research is supported by the National Natural Science Foundation of China (U25A20445, 62625608).
\end{ack}

\bibliographystyle{unsrtnat}
\bibliography{neurips_2026}

\newpage
\appendix
\input{chapters/7-apd-broader_impacts}

\input{chapters/7-apd-data}

\input{chapters/7-apd-train}

\input{chapters/7-apd-eval}

\input{chapters/7-apd-theorem}

\input{chapters/7-apd-computing_time}

\input{chapters/7-apd-case_study}

\clearpage
\input{chapters/7-apd-additional_analyses}

\end{document}

%% file: chapters/0-abs.tex
\begin{abstract}
On-policy distillation (OPD) is a standard tool for transferring teacher behavior to a smaller student, but it implicitly assumes that teacher and student predictions are comparable token by token, an assumption that fails whenever the two models tokenize the same text differently.
Under heterogeneous tokenizers, exact shared-token matching silently discards a large fraction of the teacher signal at precisely the positions where vocabularies disagree.
We propose \textbf{\underline{Sim}ple \underline{C}ross-\underline{T}okenizer OPD (SimCT)}, which restores this signal by enlarging the supervision space: alongside shared tokens, SimCT compares teacher and student over short multi-token continuations that both tokenizers can realize, leaving the OPD loss form itself unchanged.
We show that these units form the finest interface consistent with the two fixed tokenizations, and that coarser alternatives remove teacher-student distinctions that are useful for on-policy learning.
Across three heterogeneous teacher-student pairs on mathematical reasoning and code-generation benchmarks, SimCT shows consistent gains over shared-vocabulary OPD and representative cross-tokenizer baselines, with ablations supporting the benefit of recovering supervision discarded by exact shared-token matching.
Code is available at \href{https://github.com/sunjie279/SimCT}{https://github.com/sunjie279/SimCT}.
\end{abstract}

%% file: chapters/1-intro.tex
\section{Introduction}
\label{sec:intro}

Knowledge distillation transfers the behavior of a strong teacher to a smaller student and remains central to efficient LLM deployment~\citep{hinton2015distilling_kd, zhu2024compression_mc, xu2024surveykdllm_kd, yang2024kd_survey}.
For autoregressive LLMs, however, conventional offline knowledge distillation trains the student on fixed contexts rather than on states induced by its own generations~\citep{lin2020imitkd, muralidharan2024minitron_mc}, so early generation errors can compound and move the student into contexts that were rarely supervised during training~\citep{bengio2015scheduled_eb, zhang2019bridging_eb, lin2020imitkd, cideron2024sequencematch}.
On-policy distillation (OPD) addresses this train--test mismatch by querying the teacher on prefixes sampled from the student's policy, following the broader principle of on-policy imitation learning~\citep{ross2011dagger_eb, agarwal2024onpolicy_opsd, gu2024minillm_opd, ko2024distillm_opd, ko2025distillm2_opd}.
This makes OPD a natural paradigm for LLM distillation~\citep{ye2026blackbox, zhang2026lightningopd, liu2026stableopd, li2026rethinkingopd, ye2026opcd}, but it also exposes a hidden requirement for cross-model supervision: at each student-generated prefix, teacher and student predictions must be defined over comparable prediction units.

This requirement is non-trivial in realistic LLM distillation, where teacher and student models come from different families and use different tokenizers~\citep{song2026survey}.
The mismatch arises at two levels.
First, at the vocabulary level: as shown in Figure~\ref{fig:teaser}(A), modern LLM tokenizers share only a limited portion of their vocabularies across model families, so restricting supervision to tokens that appear in both vocabularies excludes teacher predictions outside this intersection.
Second, at the sequence level: the same text may be segmented into different token sequences, so the teacher's next-token distribution at a student-generated prefix may not supervise the same text unit as the student's next-token distribution~\citep{boizard2025uld, zhang2024dskd, ngo2026ctpd, singh2026bld, phan2026ctls}.
As illustrated in Figure~\ref{fig:teaser}(B), this sequence mismatch breaks a basic assumption of standard OPD, namely that teacher and student predict over the same text unit at every supervised position.

\begin{figure}[t]
 \centering
 \includegraphics[width=1.0\linewidth]{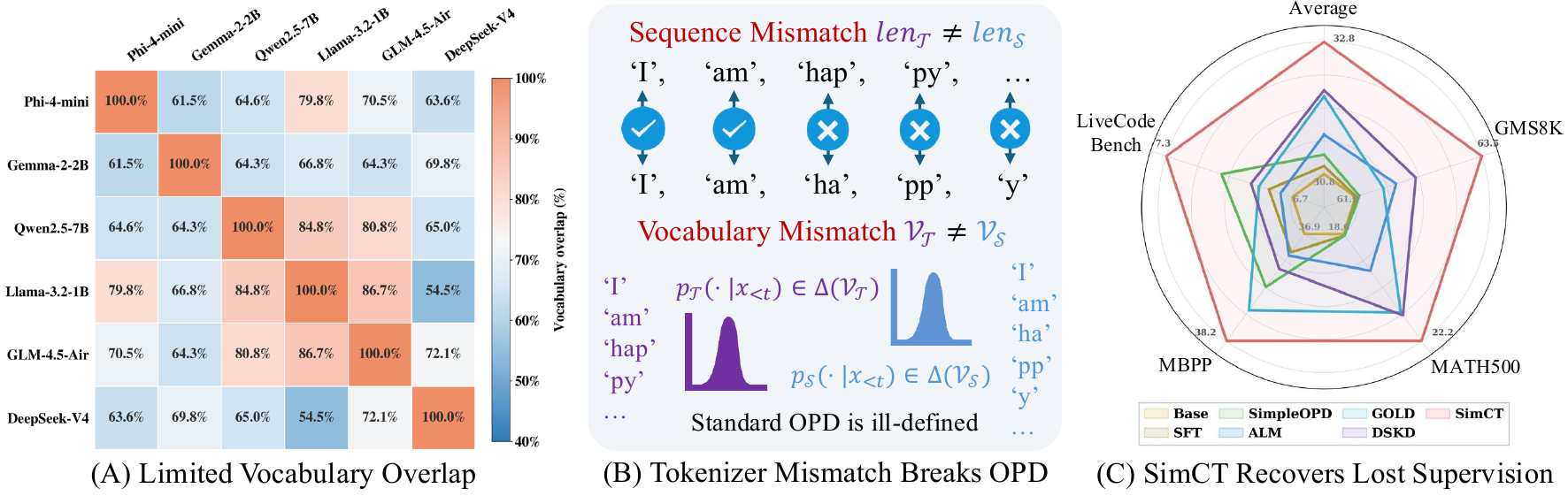}
 \caption{
 \textbf{Motivation and performance of SimCT.}
 \textbf{(A)} LLM tokenizers often share only partial vocabularies across model families, making shared-token distillation restrictive.
 \textbf{(B)} The same text can induce different token boundaries and prediction spaces, so standard token-level OPD is ill-defined.
 \textbf{(C)} SimCT constructs a common aligned supervision space and achieves the best average Pass@1 over SFT and prior cross-tokenizer baselines.
 Each axis is independently normalized for visualization; absolute results with standard deviations are reported in Table~\ref{tab:main_results_compact}.
 }
 \label{fig:teaser}
\end{figure}

Existing cross-tokenizer distillation methods use several kinds of supervision interfaces.
Distribution-alignment methods compare heterogeneous logits through optimal transport~\citep{boizard2025uld,cui2025multilevel}, while GOLD adapts distribution alignment to on-policy training through probability merging and a hybrid vocabulary loss~\citep{patino2025gold}.
Representation-transfer methods use learned cross-model projections~\citep{zhang2024dskd,zhang2025dskdv2}, introducing additional trainable parameters.
Other interfaces use chunk likelihoods~\citep{minixhofer2025alm}, BPE structure~\citep{phan2026ctls}, dynamic vocabulary mapping~\citep{chen2025cdm}, or byte-level decoders~\citep{singh2026bld}.
These approaches differ in granularity, cost, and training setting.
Our focus is a parameter-free interface that recovers supervision on student-generated prefixes using the existing tokenizer boundaries.
This raises a natural question: how simple can the supervision interface be while still recovering most of the teacher signal that tokenizer mismatch would otherwise discard?

We propose \textbf{\underline{Sim}ple \underline{C}ross-\underline{T}okenizer OPD (SimCT)}, a cross-tokenizer OPD method that changes the supervision space rather than the OPD objective.
Instead of forcing teacher and student logits to match only on identical tokens, SimCT induces a normalized preference distribution over a common supervision space and applies the standard reverse-KL OPD loss there~\citep{gu2024minillm_opd, jin2026entropy_opd}.
The space contains shared tokens and \emph{minimal aligned units} recovered from tokenizer-mismatched text.
A minimal aligned unit is the shortest non-empty span bounded by synchronized boundaries of the two fixed tokenizations, even when their token sequences differ.
SimCT requires no additional parameters, no learned projections, and no transport computation, since the supervision space is constructed directly from the two existing tokenizers. It recovers teacher signal that shared-vocabulary OPD discards. SimCT and SimpleOPD share the reverse-KL loss form; SimCT additionally uses a normalization safeguard specified in Appendix~\ref{app:training-mask}. The ablations examine the benefit of the recovered supervision within this training design.
As previewed in Figure~\ref{fig:teaser}(C), this supervision-space expansion substantially improves over existing cross-tokenizer OPD baselines.

We evaluate SimCT across three model families on mathematical reasoning and code-generation benchmarks.
SimCT consistently improves over SFT and representative cross-tokenizer OPD baselines, achieving an average relative gain of \textbf{6.0\%} over SFT with modest additional computational cost relative to shared-vocabulary OPD.
Ablations further explain where the gains come from: increasing recovered tokenizer-mismatched text supervision improves performance, whereas coarsening minimal aligned units removes KL supervision signal and reduces downstream gains.
These trends support our theoretical analysis, which characterizes SimCT as the finest boundary-consistent supervision interface jointly expressible by the teacher and student tokenizers, and shows that coarser interfaces can erase useful within-text teacher--student distinctions.
Overall, our results identify supervision-space construction as a key design choice in cross-tokenizer OPD, with minimal aligned units offering a simple way to recover teacher supervision otherwise lost under exact token overlap.

%% file: chapters/2-related_work.tex
\section{Related Work}
\label{sec:related_work}

\paragraph{Knowledge distillation and on-policy distillation.}
Knowledge distillation transfers teacher behavior to a smaller student and has become a standard tool for efficient LLM deployment~\citep{hinton2015distilling_kd,xu2024surveykdllm_kd}, with early work treating it as offline compression on fixed corpora~\citep{sanh2019distilbert_kd,jiao2020tinybert_kd,muralidharan2024minitron_mc}.
Because fixed training contexts can differ from the prefixes the student visits at inference~\citep{bengio2015scheduled_eb,lin2020imitkd}, on-policy distillation (OPD) instead queries the teacher on student-generated prefixes~\citep{lu2025onpolicydistillation, agarwal2024onpolicy_opsd}, while related distillation methods refine the objective and data use~\citep{gu2024minillm_opd,ko2024distillm_opd,ko2025distillm2_opd,ye2026blackbox,wu2025akl,luong2026drkl,kim2025csd}.
A complementary line analyzes \emph{when} and \emph{where} OPD provides signal: distributional compatibility governs success or failure at the sequence level~\citep{li2026rethinkingopd}, while token-level studies show that supervision quality varies sharply across positions~\citep{huang2025selectkd,xu2026tip_opd}.
Recent OPD methods also refine distinct components of the supervision pipeline.
TrustMOPD allocates token-level supervision among specialist teachers using label-free reliability estimates~\citep{sun2026trustmopd}; influence-directed distillation controls the effect of updates on sampled-token probabilities~\citep{sun2026idaopd}, and SOD adapts supervision at the reasoning-step level~\citep{sun2026sod}.
Guided-OPD changes the rollout distribution through curriculum-scheduled teacher guidance in multi-turn interactions~\citep{sun2026guidedopd}, whereas Data-Free OPD generates its own training questions to reduce dependence on external prompt corpora~\citep{sun2026dfopd}.
EasyOPD provides modular infrastructure for studying supervision forms, tokenizer compatibility, and optimization granularity~\citep{sun2026easyopd}.
These directions address teacher selection, signal weighting, rollout construction, data sourcing, and implementation. SimCT instead addresses the comparison units needed when tokenizers differ. These complementary design choices suggest extensions of cross-tokenizer OPD.

\paragraph{Cross-tokenizer knowledge distillation.}
When teacher and student tokenize the same text differently, prior work bridges the vocabulary gap at one of several interfaces.
Optimal-transport methods align heterogeneous logits directly, either between token distributions~\citep{boizard2025uld}, jointly at token and sequence levels~\citep{cui2025multilevel}, or layer-wise with chain-of-thought augmentation~\citep{le2025cot2align}.
Dual-space and representation-matching methods instead unify the output spaces themselves through cross-model attention and its variants~\citep{zhang2024dskd,tsiapali2026dskdkqm}.
Other methods construct comparison interfaces through approximate chunk likelihoods~\citep{minixhofer2025alm}, doing context-aware sequence matching with dynamic vocabulary mapping~\citep{chen2025cdm}, or using bytes as a tokenizer-agnostic interface~\citep{singh2026bld}, while related directions reuse the teacher loss as supervision~\citep{shin2025vocagnolm}, extend cross-tokenizer transfer to preference alignment~\citep{ngo2026ctpd}, or bridge tokenizer gaps in embedding-model representation space~\citep{liu2026mol_ct}.

%% file: chapters/3-method.tex
\section{Methodology}
\label{sec:method}

\subsection{Problem Setup}
\label{sec:setup}

Let $\mathcal T$ and $\mathcal S$ denote the teacher and student models, with vocabularies $\mathcal V_T$ and $\mathcal V_S$. In on-policy distillation, supervision is provided on prefixes sampled from the student policy rather than a fixed offline dataset. At step $t$, let $x_{<t}$ denote the student-generated prefix. Both models are conditioned on the same text-level context $x_{<t}$, but they predict the next token over their own vocabularies:
\begin{equation}
\label{eq:ctopd-distributions}
p_T(\cdot \mid x_{<t}) \in \Delta(\mathcal V_T),
\qquad
p_S(\cdot \mid x_{<t}) \in \Delta(\mathcal V_S).
\end{equation}
Here, $p_T(v \mid x_{<t})$ is the probability the teacher assigns to token $v \in \mathcal V_T$, and $p_S(u \mid x_{<t})$ is defined analogously for student token $u \in \mathcal V_S$. We write $\Delta(\mathcal V)$ for the probability simplex over $\mathcal V$.
When teacher and student share a tokenizer, the two distributions in Eq.~\eqref{eq:ctopd-distributions} are defined over a common prediction space, and standard OPD compares them directly, e.g., by minimizing a token-level KL:
\begin{equation}
\label{eq:standard-opd}
\mathcal L_{\mathrm{OPD}}(x_{<t})=\mathrm{KL}\!\left(
p_S(\cdot \mid x_{<t})\,\|\, 
p_T(\cdot \mid x_{<t})\right).
\end{equation}
This objective is well defined only because each student token corresponds to the same prediction unit as the corresponding teacher token. Thus, standard OPD implicitly assumes that teacher and student probabilities live on the same simplex and can be compared token by token.

Cross-tokenizer OPD breaks this assumption in two coupled ways. First, the vocabularies differ, $\mathcal V_T \neq \mathcal V_S$, so the teacher and student distributions in Eq.~\eqref{eq:ctopd-distributions} are supported on different probability spaces. Second, the same text may be segmented into different token sequences, so textually equivalent texts do not necessarily admit a one-to-one token alignment. We refer to these obstacles as \emph{vocabulary mismatch} and \emph{sequence mismatch}, respectively. The goal of cross-tokenizer OPD is to construct a common supervision interface in which teacher feedback remains comparable to the student's next-token prediction on student-generated prefixes.

\subsection{Generalized OPD over a Common Supervision Space}
\label{sec:gopd}

Cross-tokenizer OPD requires comparing two next-token distributions that are defined over different vocabularies and may correspond to different text segmentations. 
We handle this by replacing the shared-vocabulary assumption with a shared set of prediction units. Figure~\ref{fig:framework}(A) sketches the next-token form of this interface; the formulation below extends scoring to multi-token candidates. 
Specifically, we introduce a \emph{common supervision space} $\mathcal U$ and scoring maps. Let $\mathcal P_M(h)=\{p_M(\cdot\mid h,z):z\in\mathcal V_M^*\}$ denote model $M$'s autoregressive conditional family after prefix $h$, including the empty continuation $z$. The maps act on this family:
\begin{equation}
\label{eq:projection-maps}
\Pi_T:\mathcal P_T(h)\mapsto q_T(\cdot\mid h)\in\Delta(\mathcal U),
\qquad
\Pi_S:\mathcal P_S(h)\mapsto q_S(\cdot\mid h)\in\Delta(\mathcal U),
\end{equation}
with the candidate set held fixed when scoring. For multi-token units, access to conditional probabilities inside each candidate is required; a single next-token probability vector is insufficient. 
Given a student-generated prefix $x_{<t}$, these maps induce comparable distributions
\begin{equation}
\label{eq:projected-distributions}
q_T(\cdot \mid x_{<t})
=
\Pi_T\!\left(\mathcal P_T(x_{<t})\right),
\qquad
q_S(\cdot \mid x_{<t})
=
\Pi_S\!\left(\mathcal P_S(x_{<t})\right).
\end{equation}
Distillation can then be performed directly in $\mathcal U$:
\begin{equation}
\label{eq:gopd-objective}
\mathcal L_{\mathrm{gOPD}}(x_{<t})
=
\mathcal D\!\left(
q_S(\cdot \mid x_{<t}),
q_T(\cdot \mid x_{<t})
\right),
\end{equation}
where $\mathcal D$ denotes a distillation divergence over the common supervision space.

This formulation recovers standard OPD as a special case. 
When teacher and student share a tokenizer, we can set $\mathcal U=\mathcal V_T=\mathcal V_S$ and $\Pi_M(\mathcal P_M(h))=p_M(\cdot\mid h)$, reducing Eq.~\eqref{eq:gopd-objective} to the token-level objective in Eq.~\eqref{eq:standard-opd}. 
In the cross-tokenizer setting, however, the central design choice is how to construct $\mathcal U$ so that it preserves useful teacher supervision while remaining compatible with student predictions.
A direct baseline is to use only the shared vocabulary 
\(\mathcal U_{\mathrm{shared}}=\mathcal V_T\cap\mathcal V_S\), 
and restrict both distributions to this subspace. 
We refer to this baseline as \textbf{SimpleOPD}. 
SimpleOPD is a natural but conservative instantiation of Eq.~\eqref{eq:gopd-objective}: it avoids comparing tokens that are not shared by both vocabularies, 
but it discards supervision signal from texts that are semantically comparable yet represented by different token segments. 
Thus, SimpleOPD fails not because generalized OPD is inadequate, but because exact token overlap defines an overly narrow supervision space.
\begin{figure}[t]
 \centering
 \includegraphics[width=1.0\linewidth]{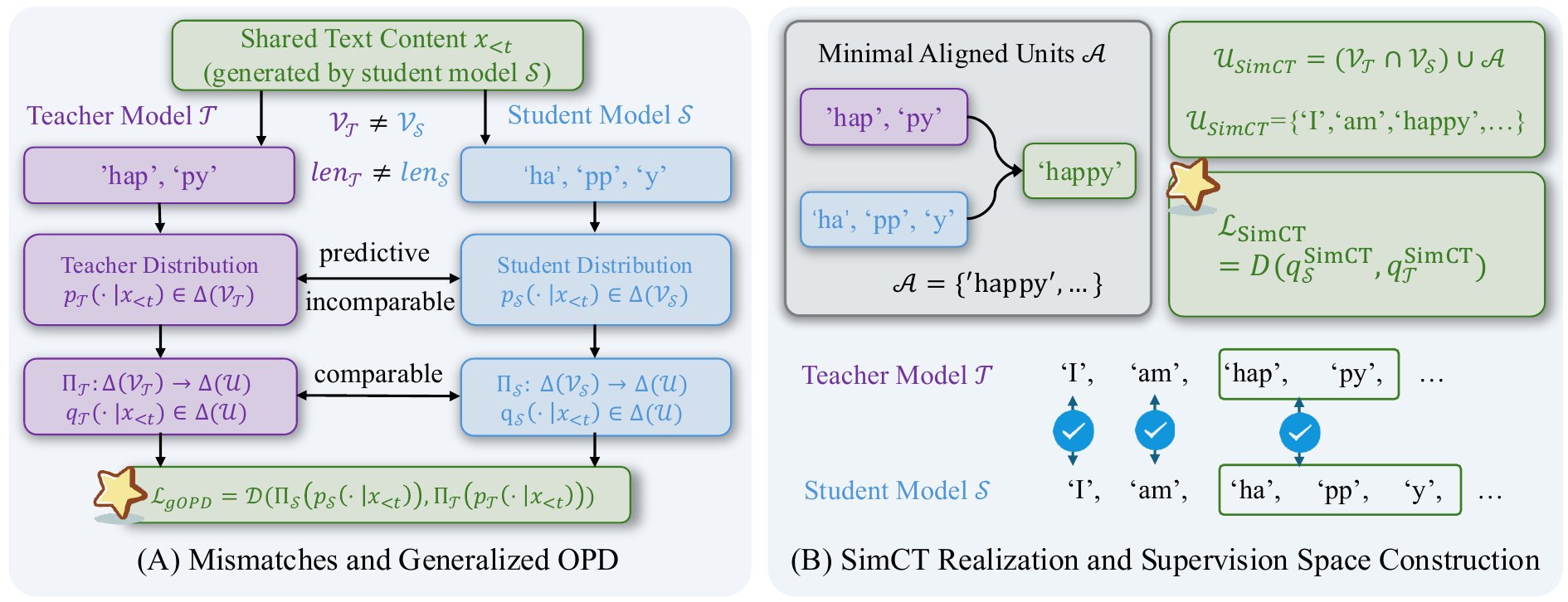}
 \caption{
 \textbf{Framework of SimCT.}
 \textbf{(A)} Teacher and student distributions are mapped to a common supervision space $\mathcal U$ under the same student-generated prefix. The panel shows the next-token form of the interface; for multi-token candidates, Eqs.~\eqref{eq:projection-maps}--\eqref{eq:projected-distributions} define scoring through the autoregressive conditional families. 
 \textbf{(B)} SimCT builds $\mathcal U_{\mathrm{SimCT}}=(\mathcal V_T\cap\mathcal V_S)\cup\mathcal A$ by adding minimal aligned units $\mathcal A$, then applies the original OPD loss on the induced supervision distributions.
}
\label{fig:framework}
\end{figure}

\subsection{SimCT: Recovering Supervision via Minimal Aligned Units}
\label{sec:simct}
To expand the common supervision space that goes beyond exact token overlap, we propose \textbf{Simple Cross-Tokenizer OPD (SimCT)}. 
Its central principle is to recover teacher supervision whenever teacher and student predictions refer to the same underlying text, even if the two tokenizers realize that text with different token boundaries. 
Figure~\ref{fig:framework}(B) gives a simple example: the teacher tokenizes ``happy'' as ``hap'', ``py'', while the student tokenizes the same text as ``ha'', ``pp'', ``y''. 
Although not comparable as individual tokens, these refer to the same text, so SimCT treats it as a recoverable supervision unit.
To construct these units, SimCT considers texts where the two tokenizations disagree and forms the finest boundary-consistent text units realizable by both tokenizers.
We call these units \emph{minimal aligned units} (MAU), and denote the set of all such units by \(\mathcal A\).
SimCT then defines its supervision space as
\begin{equation}
\label{eq:simct-space}
\mathcal U_{\mathrm{SimCT}}
=
(\mathcal V_T\cap\mathcal V_S)\cup \mathcal A .
\end{equation}
The shared-vocabulary term keeps supervision that is already comparable on 1:1-aligned shared tokens, while \(\mathcal A\) recovers supervision from texts that tokenize differently but are comparable after mapping to common aligned units.

We now define the SimCT scoring maps \(\Pi_M^{\mathrm{SimCT}}\) for \(M\in\{\mathcal T,\mathcal S\}\).
SimCT first assigns each unit \(u\in\mathcal U_{\mathrm{SimCT}}\) a continuation score and then normalizes over the candidate supervision set.
For a 1:1-aligned shared token \(u\in\mathcal V_T\cap\mathcal V_S\), we use the original next-token log-probability,
\(s_M(u\mid x_{<t})=\log p_M(u\mid x_{<t})\).
For a minimal aligned unit \(u\in\mathcal A\), let \(\tau_M(u)=(v_1,\ldots,v_k)\) be its tokenizer-\(M\) realization, and define
\begin{equation}
\label{eq:simct-unit-score}
s_M(u\mid x_{<t})
=
\frac{1}{k}\log p_M(u\mid x_{<t})
=
\frac{1}{k}
\sum_{j=1}^{k}
\log p_M(v_j\mid x_{<t},v_{<j}).
\end{equation}
The factor \(1/k\) uses the average log-likelihood rather than the raw sequence log-likelihood, avoiding systematically penalizing units that require more tokens under a given tokenizer~\citep{wu2016google_nmt}. This normalization is a standard practical choice, not a requirement of the theoretical construction in \S\ref{sec:theory}.
The normalized supervision distribution over the candidate set is then
\begin{equation}
\label{eq:simct-projected-distribution}
q_M^{\mathrm{SimCT}}(u\mid x_{<t})
=
\frac{\exp(s_M(u\mid x_{<t}))}
{\sum_{u'\in\mathcal U_{\mathrm{SimCT}}}
\exp(s_M(u'\mid x_{<t}))},
\end{equation}
which we denote as \(q_M^{\mathrm{SimCT}}=\Pi_M^{\mathrm{SimCT}}(\mathcal P_M)\).
Equation~\eqref{eq:simct-projected-distribution} defines a distribution over the finite candidate supervision set. Its probabilities are obtained by normalizing continuation scores: shared tokens retain their next-token log-probabilities, while multi-token MAUs use model-specific length-normalized continuation likelihoods. This candidate-score normalization and the original model's next-token probability measure operate on distinct spaces. 
Among retained shared tokens, the common normalizer preserves probability ratios. Multi-token MAU scores depend on the conditional factors and token counts of each model's realization; Appendix~\ref{app:coverage} details these scoring properties.

The SimCT loss is the generalized OPD loss instantiated on this recovered supervision space:
\begin{equation}
\label{eq:simct-objective}
\mathcal L_{\mathrm{SimCT}}(x_{<t})
=
\mathcal D\!\left(
q_S^{\mathrm{SimCT}}(\cdot\mid x_{<t}),
q_T^{\mathrm{SimCT}}(\cdot\mid x_{<t})
\right).
\end{equation}
SimCT performs OPD over the next continuation units rather than tokenizer-specific next tokens.
For a multi-token unit, Eq.~\eqref{eq:simct-unit-score} scores the candidate continuation under the current prefix \(x_{<t}\) through its autoregressive factors \(p_M(v_j\mid x_{<t},v_{<j})\), using only the prefix and the candidate's preceding tokens as conditioning. Candidate construction can use the already sampled student rollout, as in Algorithm~\ref{alg:simct}; it does not require an external ground-truth continuation.
Thus, SimCT preserves the on-policy training form while changing the supervision interface from exact token overlap to aligned continuation units.
For all main SimCT experiments, we mask positions where the reciprocal of the teacher score normalizer in Eq.~\eqref{eq:simct-projected-distribution} exceeds $2$; implementation details are in Appendix~\ref{app:training-mask}. The procedure is summarized in Algorithm~\ref{alg:simct}.

\subsection{Properties of the SimCT Supervision Space}
\label{sec:theory}

We present two theoretical results that clarify the design of the SimCT supervision space.
First, minimal aligned units provide the finest comparable text units consistent with the two fixed tokenizations.
Second, merging these units into coarser supervision units can remove teacher--student distinctions that are visible at the minimal-unit level.
Formal definitions, assumptions, and proofs are provided in Appendix~\ref{apd:theory}.

\begin{theorem}[Minimal aligned units]
\label{thm:canonical-space}
For a text region with fixed teacher and student tokenizations, SimCT constructs a partition into minimal aligned units.
Each unit is a union of consecutive complete tokens under both fixed tokenizations.
No unit admits a strictly finer non-empty decomposition with this property.
\end{theorem}

Theorem~\ref{thm:canonical-space} shows that SimCT constructs the finest boundary-consistent interface on which teacher and student predictions can be compared, making the supervision space more expressive than exact shared-token matching while avoiding unconstrained character-level comparison.

\begin{proposition}[Signal loss under coarser supervision spaces]
\label{prop:simct-consistency}
For a fixed student-generated prefix, let \(q_T^{\min}\) and \(q_S^{\min}\) be teacher and student distributions over SimCT minimal aligned units, with \(q_T^{\min}(u)>0\) whenever \(q_S^{\min}(u)>0\).
Let \(q_T^{\mathcal C}\) and \(q_S^{\mathcal C}\) be the distributions obtained by a non-trivial coarsening \(\mathcal C\), which sums probabilities inside each coarse unit.
Then
\begin{equation}
    \mathrm{KL}\!\left(q_S^{\min}\,\|\,q_T^{\min}\right)
    \ge
    \mathrm{KL}\!\left(q_S^{\mathcal C}\,\|\,q_T^{\mathcal C}\right).
\end{equation}
The gap is exactly the within-unit teacher--student discrepancy removed by the coarsening.
\end{proposition}

Proposition~\ref{prop:simct-consistency} motivates SimCT to preserve minimal aligned units rather than coarser units. In its idealized distributional setting, coarsening removes the within-unit teacher--student discrepancy that provides a fine-grained distillation signal; Section~\ref{sec:exp-space_ablation} measures this erased signal by the KL gap over student-generated prefixes. Together with Theorem~\ref{thm:canonical-space}, this supports using the finest tokenizer-compatible supervision units. The practical SimCT objective in Eq.~\eqref{eq:simct-projected-distribution} implements this principle with normalized continuation scores over a finite candidate set. The KL result applies to distributions coarsened by summing probabilities within fixed groups. Concatenating strings and recomputing continuation scores defines a separate transformation.

%% file: chapters/4-experiment.tex
\section{Experiments}
\label{sec:exp}

\subsection{Experimental Setup}
\label{sec:exp-setup}

\paragraph{Models and pairs.}
We construct three teacher--student pairs that span all three model families and tokenizers: Qwen2.5-7B-Instruct~\citep{qwen25} $\to$ Phi-4-mini-Instruct~\citep{phi4mini}, Qwen $\to$ Gemma-2-2B-IT~\citep{gemma2}, and Phi $\to$ Gemma. 
All pairs follow the realistic large-to-small distillation direction. Two distinct models (Qwen and Phi) serve as teachers and two distinct models (Phi and Gemma) serve as students across pairs, so our results do not hinge on any single teacher or student choice.

\paragraph{Data and warm-start.}
We construct teacher-specific corpora to warm-start the students before OPD. Prompts are drawn from GSM8K~\citep{cobbe2021gsm8k}, Orca-Math~\citep{mitra2024orcamath}, OpenMathInstruct-1~\citep{toshniwal2024openmathinstruct}, and MATH excluding MATH-500~\citep{hendrycks2021math} for math, and OpenCodeInstruct~\citep{ahmad2025opencodeinstruct}, KodCode~\citep{xu2025kodcode}, TACO~\citep{li2023taco}, and CodeContests~\citep{li2022alphacode} for code. For each prompt, the teacher samples 8 responses (temperature $0.6$, top-$p=0.95$); we keep one per prompt via task-specific correctness checks where available and quality filters otherwise, yielding a 10K-example SFT corpus for each teacher. Within each teacher--student pair and training seed, all OPD methods start from the same corresponding SFT checkpoint. Full composition, filtering, and license details are in Appendix~\ref{apd:data}.

\paragraph{On-policy distillation.}
Starting from the warm-start checkpoint, all methods follow a common OPD loop with method-specific supervision and loss safeguards: at each step the student samples a prefix under its current policy, the teacher is queried on this prefix to produce next-token supervision, and the student is updated under each method's loss. Training prompts, prefix-generation settings, optimizer, schedule, training epochs, and hardware are held identical across methods. Hyperparameters are in Appendix~\ref{apd:train}.

\paragraph{Baselines.}
We compare SimCT against four baselines that span the main families of cross-tokenizer distillation. SimpleOPD performs exact shared-vocabulary OPD, supervising only the tokens that appear in both vocabularies. DSKD~\citep{zhang2024dskd} unifies teacher and student output spaces through a learned cross-model projection, requiring additional trainable parameters. ALM~\citep{minixhofer2025alm} matches approximate likelihoods over tokenizer-induced chunks, providing supervision at the chunk level rather than the token level. GOLD~\citep{patino2025gold} extends universal logit distillation~\citep{boizard2025uld} to on-policy training with text-aligned probability merging across tokenizer boundaries.

\paragraph{Benchmarks.}
We evaluate mathematical reasoning on GSM8K~\citep{cobbe2021gsm8k} and MATH-500~\citep{lightman2023math500}, and code generation on MBPP~\citep{austin2021program_mbpp} and LiveCodeBench-v6~\citep{jain2024livecodebench}. For each SFT or OPD configuration, we train independently with 5 seeds and evaluate each resulting checkpoint once under zero-shot prompting, reporting mean pass@1 and standard deviation across the 5 runs; full templates, decoding settings, and scoring details are in Appendix~\ref{apd:eval}.

\subsection{Overall Effectiveness}
\label{sec:exp-effective}

\input{tables/main-ctopd}

Table~\ref{tab:main_results_compact} reports the main results under the cross-tokenizer OPD setting.
Within each pair and training seed, all distillation methods start from the same SFT-warm-started student checkpoint and follow the same common training protocol.
As reported in Appendix~\ref{app:compute}, SimCT only slightly increases the wall-clock cost over SimpleOPD and is faster than the evaluated ALM, GOLD, and DSKD implementations.
The headline $6.0\%$ relative gain is the unweighted mean of the three pair-level improvements over SFT: $6.0\%$, $7.3\%$, and $4.7\%$ after rounding. Each pair-level score first averages the four benchmark pass@1 values, so this summary gives equal weight to benchmarks and to model pairs.
From these results, we draw the following observations:

\begin{itemize}[leftmargin=*]

\item \textbf{Obs 1: SimCT improves cross-tokenizer OPD consistently across teachers, students, and task domains.}
SimCT achieves the best average performance for all three teacher--student pairs and obtains the strongest result on every evaluated benchmark.
The gains hold when Qwen2.5-7B-Instruct distills to two different student families, and also when Phi-4-mini-instruct serves as the teacher for Gemma-2-2B-it.
They also appear on both mathematical reasoning and code-generation benchmarks.
This consistency suggests that SimCT is not tied to a particular teacher, student, or task type, but provides a robust supervision interface for cross-tokenizer OPD.

\item \textbf{Obs 2: Exact shared-token supervision is too conservative under tokenizer mismatch.}
SimpleOPD uses the most direct common supervision space: the exact shared vocabulary.
However, it improves only marginally over the SFT warm start across the three distillation settings.
This indicates that shared tokens preserve only a limited portion of useful teacher feedback when teacher and student tokenizers differ.
By adding minimal aligned units, SimCT recovers supervision from text continuations that are comparable but not represented as identical tokens, leading to consistently stronger performance than SimpleOPD.

\item \textbf{Obs 3: Prior cross-tokenizer interfaces are useful but less reliable across settings.}
ALM, GOLD, and DSKD often improve over SimpleOPD, confirming that cross-tokenizer supervision beyond exact token overlap is beneficial.
However, their relative rankings vary across teacher--student pairs and benchmarks: no existing baseline gives a uniformly strongest interface.
This suggests that making heterogeneous predictions comparable is necessary but not sufficient for cross-tokenizer OPD.
The supervision interface must also preserve local teacher--student distinctions needed for on-policy next-token feedback, which SimCT addresses through minimal aligned units.

\end{itemize}

Overall, these observations support the central design choice of SimCT: keep the OPD training paradigm unchanged, but improve the supervision space in which teacher feedback is expressed.
Minimal aligned units provide a more robust supervision interface than exact shared tokens or prior cross-tokenizer mappings, leading to consistent gains in realistic cross-tokenizer OPD settings. 
Appendix~\ref{app:case_study} presents a case study with math and code examples illustrating local reasoning or implementation errors in the generated responses.

\begin{figure}[!t]
 \centering
 \includegraphics[width=1.0\linewidth]{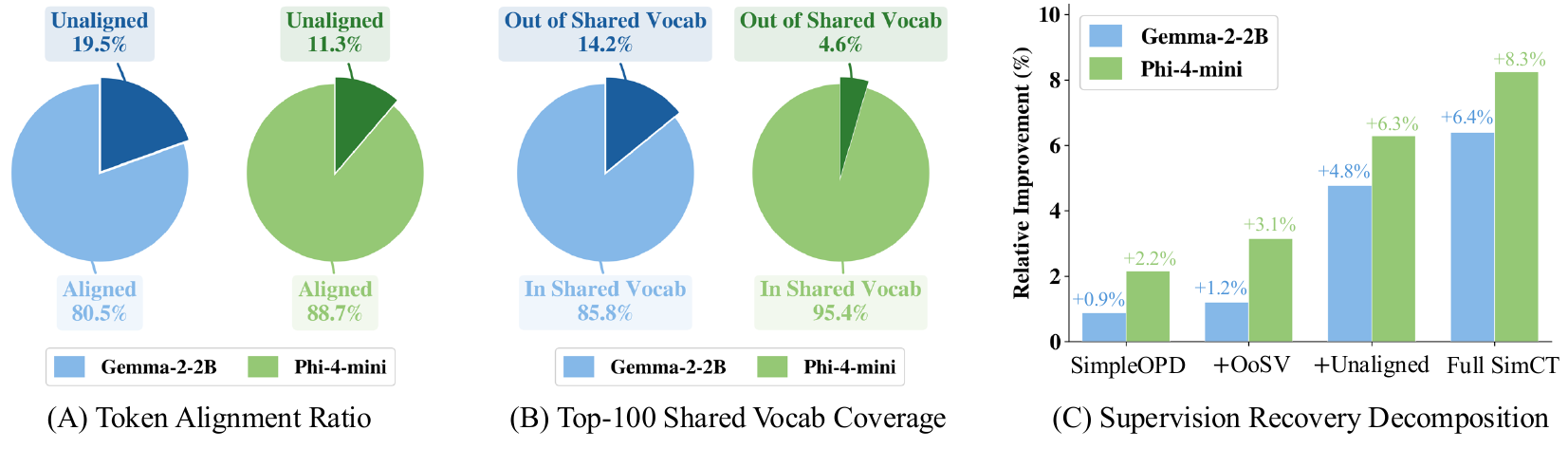}
 \caption{
 \textbf{Supervision recovery under tokenizer mismatch.}
 \textbf{(A)} Many teacher and student tokens are not aligned one-to-one, revealing supervision loss from sequence mismatch.
 \textbf{(B)} Even at aligned positions, high-probability teacher predictions may fall outside the shared vocabulary.
 \textbf{(C)} Recovering either missing source improves over Base, while Full SimCT performs best by recovering both in a common aligned-unit space.
 \textbf{+Out of Shared Vocab (+OoSV)} adds non-shared vocabulary supervision at aligned positions; \textbf{+Unaligned} adds supervision from non-1:1 aligned units. Panel (C) reports independent ablation runs, with relative improvement over the corresponding unadapted Base student.
 }
\label{fig:mismatch_analysis}
\end{figure}

\subsection{Supervision Recovery under Tokenizer Mismatch}
\label{sec:exp-supervision_recovery}

We next ask whether SimCT improves cross-tokenizer OPD by recovering teacher supervision that is discarded under tokenizer mismatch.
Figure~\ref{fig:mismatch_analysis} decomposes the missing supervision, while Figure~\ref{fig:ablation} tests whether recovering this supervision leads to downstream gains.
The downstream results in these figures come from independent ablation runs. Their relative improvements use the corresponding unadapted Base student as the reference; Table~\ref{tab:main_results_compact} reports the main training runs, with gains measured relative to SFT.
From these analyses, we draw the following observations:

\begin{itemize}[leftmargin=*]

\item \textbf{Obs 4: Tokenizer mismatch removes useful teacher feedback through both sequence mismatch and vocabulary mismatch.}
Figure~\ref{fig:mismatch_analysis}(A) shows that a non-trivial fraction of teacher- and student-side tokens cannot be aligned one-to-one. 
This indicates that standard token-level matching fails to use teacher feedback on units that correspond to the same underlying text but are realized with different tokenizer boundaries.
Figure~\ref{fig:mismatch_analysis}(B) reveals a complementary failure mode: even at token-to-token-aligned positions, many high-probability teacher predictions fall outside the shared vocabulary.
Together, these results show that SimpleOPD is conservative in two distinct ways: it discards supervision from non-1:1 aligned texts, and it also removes non-shared vocabulary predictions even when the current positions are aligned.

\item \textbf{Obs 5: The discarded signals are beneficial rather than incidental tokenizer artifacts.}
Figure~\ref{fig:mismatch_analysis}(C) directly tests whether the missing signals identified above improve learning.
Starting from SimpleOPD, \textbf{+Out of Shared Vocab} adds non-shared vocabulary supervision at token-to-token-aligned positions, while \textbf{+Unaligned} adds supervision from tokenizer-mismatched units that cannot be aligned one-to-one.
Both variants improve over SimpleOPD for both student models, showing that the discarded supervision signal contains useful distillation information.
Full SimCT achieves the largest improvement because it recovers both sources through a unified aligned-unit supervision space.

\item \textbf{Obs 6: Recovering more mismatch unit supervision consistently improves downstream performance.}
Figure~\ref{fig:ablation}~(Left) further varies the fraction of tokenizer-mismatched unit supervision incorporated into training.
For both Gemma-2-2B and Phi-4-mini, performance improves as more recovered supervision is used.
This trend strengthens the mechanism suggested by Figure~\ref{fig:mismatch_analysis}: mismatch unit supervision is not merely recoverable, but provides a stable learning signal when mapped into a common supervision space.
SimCT exploits this signal by aligning tokenizer-specific continuations into minimal aligned units, rather than discarding them under exact token overlap.

\end{itemize}

These findings explain why SimpleOPD yields only limited gains in Table~\ref{tab:main_results_compact}. 
Shared-vocabulary distillation preserves only the subset of teacher feedback that is already comparable in the original token spaces.
SimCT expands the supervision space to reincorporate supervision signal excluded by shared-vocabulary matching, thereby recovering useful teacher supervision while keeping the original OPD objective unchanged.
Across $263{,}817$ rollout positions, MAUs increase mean represented teacher next-token support from $95.6\%$ to $99.5\%$ (Appendix~\ref{app:coverage}). Additional overlap controls, same-tokenizer references, and position-level analyses are reported in Appendices~\ref{app:overlap-reference} and~\ref{app:position-analysis}.

\begin{figure}[!t]
 \centering
 \includegraphics[width=1.0\linewidth]{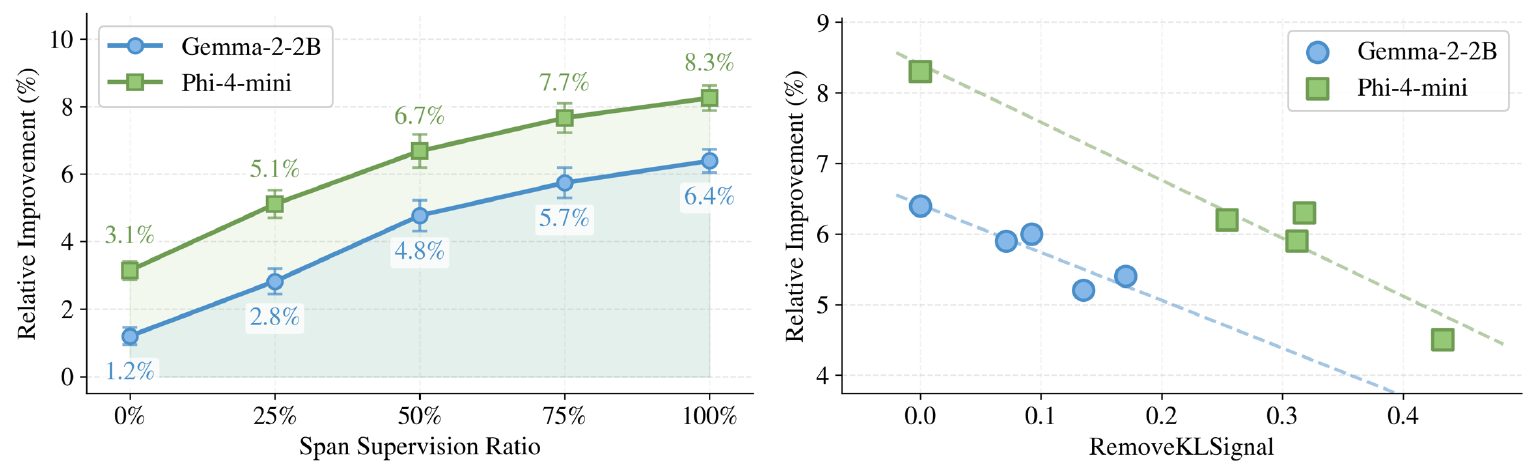}
 \caption{
 \textbf{Ablation on recovered unit supervision and aligned-unit coarsening.}
 \textbf{(Left)} Recovering more mismatch-unit supervision consistently improves both students.
 \textbf{(Right)} Larger removed within-span KL signal is generally associated with smaller gains. These are independent ablation runs; improvement is relative to the corresponding Base student. Error bars show standard deviations over 5 independent training runs, with one evaluation per checkpoint. The 0\% MAU point on the left retains OoSV supervision and is therefore SimpleOPD + OoSV.
}
\label{fig:ablation}
\end{figure}

\subsection{Ablation on Supervision-Space Construction}
\label{sec:exp-space_ablation}

The previous analysis shows that tokenizer-mismatched units contain useful teacher supervision. We now ask whether the \emph{granularity} of the recovered supervision also matters.
Starting from the SimCT supervision space, we construct coarsened variants by merging adjacent minimal aligned units into larger units.
Let \(q_T^{\min}\) and \(q_S^{\min}\) denote the teacher and student distributions over minimal aligned units.
For a coarsening \(\mathcal C\), the corresponding coarse distributions aggregate the probability mass of all minimal units inside each coarse unit:
\begin{equation}
\label{eq:coarsened-distribution}
q_T^{\mathcal C}(c\mid h)
=
\sum_{u\in c} q_T^{\min}(u\mid h),
\qquad
q_S^{\mathcal C}(c\mid h)
=
\sum_{u\in c} q_S^{\min}(u\mid h),
\end{equation}
where \(h\) is the student-generated prefix.
This aggregation preserves the total probability assigned to each coarse unit, but collapses the relative teacher--student preferences among its constituent minimal units.
We quantify the erased supervision signal as:
\begin{equation}
\label{eq:removed-kl-signal}
\Delta_{\mathcal C}
=
\mathbb E_h
\left[
\mathrm{KL}\!\left(q_S^{\min}(\cdot\mid h)\,\|\,q_T^{\min}(\cdot\mid h)\right)
-
\mathrm{KL}\!\left(q_S^{\mathcal C}(\cdot\mid h)\,\|\,q_T^{\mathcal C}(\cdot\mid h)\right)
\right].
\end{equation}
Thus, \(\Delta_{\mathcal C}\) measures the within-unit KL signal removed when minimal aligned units are merged.

\begin{itemize}[leftmargin=*]
\item \textbf{Obs 7: Coarsening removes useful local supervision.}
Figure~\ref{fig:ablation}~(Right) shows a negative relationship between removed KL signal and downstream improvement: coarsenings with larger \(\Delta_{\mathcal C}\) tend to yield weaker gains for both Gemma-2-2B-it and Phi-4-mini.
This indicates that coarsening does more than simplify the supervision space; it erases local teacher--student preferences useful for OPD.
The result supports Proposition~\ref{prop:simct-consistency} and explains why SimCT keeps minimal aligned units instead of merging them into coarser units.
\end{itemize}

Overall, this ablation shows that SimCT benefits not only from recovering supervision under tokenizer mismatch, but also from preserving it at the finest granularity jointly expressible by both tokenizers.
An additional within-unit scoring ablation favors token-length normalization over unnormalized and byte-normalized scores on Qwen2.5 $\to$ Gemma-2 (Appendix~\ref{app:normalization-ablation}).

%% file: tables/main-ctopd.tex
\begin{table}[t]
\centering
\small
\setlength{\tabcolsep}{6.5pt}
\newcommand{\std}[1]{{\,\scalebox{0.7}{\ensuremath{\pm}#1}}}
\caption{
Main results in the cross-tokenizer OPD setting.
All distillation methods start from the same SFT-warm-started student checkpoint within each teacher--student pair and seed, and follow the same common training protocol.
For SFT and OPD, we report mean pass@1 and standard deviation over 5 independent training runs, evaluating each checkpoint once. Base models use 5 decoding seeds.
The Average column is the arithmetic mean of the four displayed benchmark means, rounded to two decimal places.
Best and second-best student results within each teacher--student pair are shown in \textbf{bold} and \underline{underlined}.
}
\vspace{1mm}
\label{tab:main_results_compact}
\begin{tabular}{ccccccc}
\toprule
\multirow{2.5}{*}{\centering Model} & \multirow{2.5}{*}{\centering Method} & \multicolumn{2}{c}{Math} & \multicolumn{2}{c}{Code} & \multirow{2.5}{*}{\centering Average} \\
\cmidrule(lr){3-4} \cmidrule(lr){5-6}
& & GSM8K & MATH500 & MBPP & LCB-v6 & \\
\midrule
Gemma-2-2B-it & Base
& 61.85\std{0.25} & 17.96\std{0.38} & 36.92\std{0.30} & \phantom{0}6.65\std{0.23} & 30.85 \\
Phi-4-mini-instruct & Base
& 85.75\std{0.31} & 54.56\std{0.49} & 55.36\std{0.39} & \phantom{0}7.92\std{0.30} & 50.90 \\
Qwen2.5-7B-Instruct & Base
& 91.43\std{0.33} & 66.20\std{0.54} & 65.80\std{0.43} & 13.71\std{0.37} & 59.29 \\
\midrule

\multirow{6}{2.7cm}{\centering
\shortstack[c]{
Qwen2.5-7B-Instruct\\
{\footnotesize teacher}\\
$\downarrow$\\
Gemma-2-2B-it\\
{\footnotesize student}
}}
& \sftc{SFT}
& \sftc{61.86\std{0.22}} & \sftc{18.00\std{0.35}} & \sftc{37.14\std{0.28}} & \sftc{\phantom{0}6.77\std{0.20}} & \sftc{30.94} \\
& SimpleOPD
& 61.90\std{0.25} & 18.00\std{0.40} & 37.56\std{0.31} & \phantom{0}\underline{7.01}\std{0.22} & 31.12 \\
& ALM
& 62.39\std{0.27} & 19.40\std{0.44} & 37.18\std{0.30} & \phantom{0}6.71\std{0.21} & 31.42 \\
& GOLD
& 62.22\std{0.30} & 21.04\std{0.48} & \underline{37.84}\std{0.34} & \phantom{0}6.82\std{0.23} & 31.98 \\
& DSKD
& \underline{62.65}\std{0.31} & \underline{21.16}\std{0.49} & 37.34\std{0.32} & \phantom{0}6.86\std{0.24} & \underline{32.00} \\
& \simctc{SimCT}
& \simctc{\textbf{63.53}\std{0.42}} & \simctc{\textbf{22.16}\std{0.52}} & \simctc{\textbf{38.21}\std{0.35}} & \simctc{\phantom{0}\textbf{7.29}\std{0.25}} & \simctc{\textbf{32.80}} \\
\midrule

\multirow{6}{2.7cm}{\centering
\shortstack[c]{
Qwen2.5-7B-Instruct\\
{\footnotesize teacher}\\
$\downarrow$\\
Phi-4-mini-instruct\\
{\footnotesize student}
}}
& \sftc{SFT}
& \sftc{86.06\std{0.28}} & \sftc{55.16\std{0.45}} & \sftc{55.40\std{0.36}} & \sftc{\phantom{0}8.57\std{0.28}} & \sftc{51.30} \\
& SimpleOPD
& 87.04\std{0.30} & 56.80\std{0.50} & 55.02\std{0.38} & \phantom{0}9.23\std{0.31} & 52.02 \\
& ALM
& 87.81\std{0.33} & 59.84\std{0.56} & 55.12\std{0.40} & 10.86\std{0.37} & 53.41 \\
& GOLD
& 88.32\std{0.34} & \underline{60.12}\std{0.59} & 55.08\std{0.41} & \underline{12.06}\std{0.42} & \underline{53.90} \\
& DSKD
& \underline{88.63}\std{0.34} & 59.68\std{0.55} & \underline{55.78}\std{0.40} & 11.43\std{0.39} & 53.88 \\
& \simctc{SimCT}
& \simctc{\textbf{90.22}\std{0.40}} & \simctc{\textbf{61.40}\std{0.60}} & \simctc{\textbf{56.04}\std{0.41}} & \simctc{\textbf{12.57}\std{0.44}} & \simctc{\textbf{55.06}} \\
\midrule

\multirow{6}{2.7cm}{\centering
\shortstack[c]{
Phi-4-mini-instruct\\
{\footnotesize teacher}\\
$\downarrow$\\
Gemma-2-2B-it\\
{\footnotesize student}
}}
& \sftc{SFT}
& \sftc{62.32\std{0.22}} & \sftc{18.20\std{0.36}} & \sftc{37.48\std{0.29}} & \sftc{\phantom{0}6.43\std{0.20}} & \sftc{31.11} \\
& SimpleOPD
& \underline{62.62}\std{0.25} & 19.00\std{0.42} & 37.89\std{0.32} & \phantom{0}6.43\std{0.22} & 31.49 \\
& ALM
& 62.24\std{0.25} & 19.80\std{0.45} & 37.80\std{0.33} & \phantom{0}\underline{6.63}\std{0.23} & 31.62 \\
& GOLD
& 62.52\std{0.26} & \underline{20.36}\std{0.47} & \underline{38.82}\std{0.35} & \phantom{0}6.57\std{0.24} & \underline{32.07} \\
& DSKD
& 62.37\std{0.26} & 19.80\std{0.44} & 38.64\std{0.34} & \phantom{0}6.57\std{0.23} & 31.85 \\
& \simctc{SimCT}
& \simctc{\textbf{62.92}\std{0.31}} & \simctc{\textbf{21.08}\std{0.49}} & \simctc{\textbf{39.16}\std{0.36}} & \simctc{\phantom{0}\textbf{7.14}\std{0.27}} & \simctc{\textbf{32.58}} \\
\bottomrule
\end{tabular}
\end{table}

%% file: chapters/6-limitation.tex
\section{Limitations}
\label{sec:limitation}
Our main evaluation centers on math and code reasoning with three teacher--student pairs spanning two teacher families and heterogeneous tokenizers; extending the study to additional domains, scales, and tokenizer families is left to future work. Methodologically, SimCT relies on mapping tokenizer mismatches to minimal aligned units, which is natural for standard BPE-style tokenizers and may require adaptation for tokenizers with substantially different segmentation behavior. 
Finally, we target the white-box OPD setting in which teacher next-token distributions are available; extending supervision-space construction to black-box teachers that expose only samples or scalar feedback is a promising direction.

%% file: chapters/5-conclusion.tex
\section{Conclusion}

We studied cross-tokenizer OPD, where teacher and student policies follow OPD but use different tokenizers. 
We argued that cross-tokenizer OPD should be viewed as a supervision-space construction problem: useful teacher feedback is lost when heterogeneous token-level predictions are forced to match only on exact shared tokens. 
We introduced SimpleOPD as a conservative shared-vocabulary baseline and proposed SimCT, which expands the common supervision space with minimal aligned units while preserving the standard OPD objective. 
Our theoretical analysis characterizes these units as the finest supervision interface consistent with the two fixed tokenizations and shows that coarsening them can erase useful teacher--student distinctions. 
Across model families and math/code benchmarks, SimCT consistently improves over SimpleOPD and representative cross-tokenizer distillation baselines. 
These results suggest that changing the supervision space, rather than the OPD training paradigm, is an effective path for cross-tokenizer OPD.

%% file: chapters/7-apd-broader_impacts.tex
\section{Broader Impacts}
\label{apd:broader_impacts}
SimCT may improve the accessibility of language-model distillation by enabling more knowledge transfer across models with different tokenizers.
This can reduce the cost of training capable, smaller models and make model development more accessible.
At the same time, stronger distillation can transfer undesirable teacher behaviors, including biases, unsafe responses, and hallucinations, and may lower the cost of producing models that could be misused.
Therefore, SimCT should be paired with appropriate data filtering, safety evaluation, and alignment procedures before deployment.

%% file: chapters/7-apd-data.tex
\section{Training Data Construction and Curation}
\label{apd:data}

This appendix describes the construction of the teacher-specific cold-start SFT corpora. Within each pair and seed, the corresponding warm-start checkpoint is shared by all OPD methods. We use public math and code datasets as prompt sources, apply filtering and deduplication, generate multiple teacher responses per prompt, and select one response according to correctness, executability, and formatting criteria. The SFT training procedure that consumes this corpus is described in Appendix~\ref{apd:train}.

\subsection{Composition}

Each teacher-specific corpus contains 10{,}000 examples: 5{,}800 mathematical reasoning examples and 4{,}200 code-generation examples. The math/code split roughly matches the relative emphasis of the two task families in our evaluation suite, and within each family we draw from multiple sources to avoid overfitting to any single dataset's style.

For mathematical reasoning, we sample 1{,}800 prompts from the GSM8K~\citep{cobbe2021gsm8k} training split, 2{,}200 from Orca-Math 200K~\citep{mitra2024orcamath}, 1{,}000 from OpenMathInstruct-1~\citep{toshniwal2024openmathinstruct}, and 800 from MATH~\citep{hendrycks2021math}, the last with the MATH-500 evaluation split~\citep{lightman2023math500} excluded. The MATH subset is drawn from levels 1--4 (150/150/350/150, respectively); we exclude level 5 because its solutions frequently exceed our SFT context budget and exhibit high teacher-side variance, both of which inject noise into the warm-start signal without proportionate benefit. Excluded difficulty levels are still represented in evaluation through MATH-500.

For code generation, we sample 1{,}800 prompts from OpenCodeInstruct~\citep{ahmad2025opencodeinstruct}, 900 from KodCode~\citep{xu2025kodcode}, 900 from TACO~\citep{li2023taco}, and 600 from CodeContests~\citep{li2022alphacode}. The resulting subset covers function-level Python synthesis and competitive-programming I/O formats, with a preference toward examples that admit reliable execution-based verification.

\subsection{Prompt Filtering and Deduplication}

Before teacher generation, we filter prompts to remove noisy, duplicated, ill-formed, overly long, or difficult-to-verify instances, and we run overlap checks against held-out evaluation data using normalized text and near-duplicate matching to prevent test contamination.

For math sources, we retain text-only problems with clear statements, well-defined answer formats, and complete textual solutions. We discard examples with malformed questions, inconsistent answer formats, excessive derivation length, embedded code blocks, code-interpreter traces, or dependencies on diagrams, Asymptote figures, or other non-textual reasoning artifacts. We additionally apply problem-level near-duplicate filtering to suppress repeated template families.

For code sources, we keep only Python tasks with a well-defined solution target. When tests are available, we require compatibility with execution-based checking. We discard debugging, iterative-repair, SQL, shell, front-end, multi-file, interactive, project-dependent, and non-executable tasks. For contest-style sources (TACO, CodeContests), we further remove problems whose statements exceed our context budget or cannot be cleanly converted into standard Python generation examples.

\input{tables/apd-data-source}

\subsection{Teacher Response Generation and Selection}

For each retained prompt, Qwen2.5-7B-Instruct and Phi-4-mini-instruct separately sample 8 candidate responses with temperature $0.6$ and top-$p=0.95$, yielding a diverse but focused candidate set. We then select a single response per prompt as the SFT target.

\paragraph{Verification-based selection.}
When automatic verification is available, we keep only candidates that pass the corresponding checker. For math, we use \texttt{math\_verify}\footnote{\url{https://github.com/huggingface/math-verify}} to compare extracted final answers under symbolic equivalence (handling fractions, equivalent algebraic forms, and standard formatting variations). For code, we execute each candidate against the dataset's provided unit tests in a sandboxed Python environment with a per-test wall-clock timeout of 10 seconds and a memory cap of 128 MB. When multiple candidates pass, we keep the shortest valid response after standardized formatting normalization, which suppresses unnecessary verbosity and limits incidental teacher-specific style.

\paragraph{Quality-filter selection.}
When verification is unavailable or partial (e.g., open-ended math without a canonical final-answer field, or code prompts without supplied tests), we apply rule-based quality filters: responses must contain a non-empty final answer in the expected format (boxed final answer for math, a complete fenced Python code block with a defined entry point for code), must not exceed 4096 tokens, and must not contain unterminated reasoning, repeated boilerplate, or empty solution bodies. The full filter rule set is documented in our released curation script. Among candidates passing all filters, we again prefer the shortest.

\paragraph{Role of the public datasets.}
Public datasets are used as \emph{prompt sources only}; the supervised targets are generated and selected through the unified teacher-response pipeline above. This keeps warm-start supervision consistent with the teacher policy used during the subsequent OPD stage and prevents distributional drift between the two training phases. Source attribution and metadata are preserved per example throughout the pipeline.

\subsection{Licenses, Attribution, and Release}

Table~\ref{tab:asset_licenses} summarizes the public assets used in this work along with their roles and license terms. For each example in the released corpus we retain the source dataset, split, original identifier when available, source license, teacher model identifier, and generation configuration.

\input{tables/apd-asset-license}

Because prompt sources carry different licenses and terms, we do not assign a single license to the constructed corpus. Per-example source attribution and license metadata are retained, and downstream users are responsible for compliance with each original asset's terms; assets with non-commercial or mixed-source terms are flagged accordingly.

The public repository, \url{https://github.com/sunjie279/SimCT-}, provides the data-preparation, teacher-generation, SFT, distillation, and evaluation workflow. The corpus can be reconstructed from the original sources using the documented pipeline; redistribution of individual examples remains subject to their source terms.

%% file: tables/apd-data-source.tex
\begin{table}[t]
\centering
\small
\caption{
Composition of the cold-start SFT corpus. Public datasets are used as prompt sources; supervised targets are generated by Qwen2.5-7B-Instruct~\citep{qwen25} / Phi-4-mini-instruct~\citep{phi4mini} and selected through the response-selection pipeline described in Appendix~\ref{apd:data}.
}
\vspace{2mm}
\begin{tabular}{llr}
\toprule
Domain & Prompt source & \# Examples \\
\midrule
\multirow{5.5}{*}{Math}
& GSM8K train~\citep{cobbe2021gsm8k} & 1,800 \\
& Orca-Math 200K~\citep{mitra2024orcamath} & 2,200 \\
& OpenMathInstruct-1~\citep{toshniwal2024openmathinstruct} & 1,000 \\
& MATH, excluding MATH-500~\citep{hendrycks2021math,lightman2023math500} & 800 \\
\cmidrule(lr){2-3}
& \textit{Math subtotal} & \textit{5,800} \\
\midrule
\multirow{5.5}{*}{Code}
& OpenCodeInstruct~\citep{ahmad2025opencodeinstruct} & 1,800 \\
& KodCode~\citep{xu2025kodcode} & 900 \\
& TACO~\citep{li2023taco} & 900 \\
& CodeContests~\citep{li2022alphacode} & 600 \\
\cmidrule(lr){2-3}
& \textit{Code subtotal} & \textit{4,200} \\
\midrule
\multicolumn{2}{r}{\textbf{Total}} & \textbf{10,000} \\
\bottomrule
\end{tabular}
\label{tab:cold_start_data_composition}
\end{table}

%% file: tables/apd-asset-license.tex
\begin{table}[t]
\centering
\small
\caption{Source assets used in the main experiments. Public datasets are used as prompt sources for constructing the cold-start SFT corpus. The final supervised targets are generated by the teacher model and selected through the response-selection pipeline described in Appendix~\ref{apd:data}. We preserve source attribution and license metadata for each example.}
\vspace{2mm}
\begin{tabular}{llll}
\toprule
Asset & Category & Role in This Work & License / Terms \\
\midrule
KDFlow~\citep{zhang2026kdflow} & Code & Codebase & MIT \\
Qwen2.5-7B-Instruct~\citep{qwen25} & Model & Teacher model & Apache-2.0 \\
Gemma-2-2B-IT~\citep{gemma2} & Model & Student model & Gemma Terms of Use \\
Phi-4-mini-instruct~\citep{phi4mini} & Model & Teacher / student & MIT \\
\midrule
GSM8K~\citep{cobbe2021gsm8k} & Math data & Prompt source & MIT \\
Orca-Math 200K~\citep{mitra2024orcamath} & Math data & Prompt source & MIT \\
OpenMathInstruct-1~\citep{toshniwal2024openmathinstruct} & Math data & Prompt source & NVIDIA License \\
MATH\textsuperscript{$\dagger$}~\citep{hendrycks2021math} & Math data & Prompt source & MIT \\
\midrule
OpenCodeInstruct~\citep{ahmad2025opencodeinstruct} & Code data & Prompt source & CC BY 4.0 \\
KodCode~\citep{xu2025kodcode} & Code data & Prompt source & CC BY-NC 4.0 \\
TACO\textsuperscript{$\ddagger$}~\citep{li2023taco} & Code data & Prompt source & Apache-2.0 / mixed permissive terms \\
CodeContests\textsuperscript{$\S$}~\citep{li2022alphacode} & Code data & Prompt source & Apache-2.0 / CC BY 4.0 \\
\bottomrule
\end{tabular}
\vspace{0.2em}
\begin{minipage}{1\linewidth}
\vspace{2mm}
\footnotesize
\textsuperscript{$\dagger$} We use examples from the original MATH dataset~\citep{hendrycks2021math} after excluding MATH-500~\citep{lightman2023math500} evaluation instances. The source MATH dataset is released under the MIT license. 
\textsuperscript{$\ddagger$} TACO~\citep{li2023taco} is released under Apache-2.0 by its authors, while its official license note states that the dataset may include content under MIT and CC BY 4.0 terms.
\textsuperscript{$\S$} CodeContests~\citep{li2022alphacode} licenses code under Apache-2.0 and non-code materials under CC BY 4.0; additional third-party source terms may apply.
\end{minipage}
\vspace{2mm}
\label{tab:asset_licenses}
\end{table}

%% file: chapters/7-apd-train.tex
\section{Training Details}
\label{apd:train}

For each student model, training proceeds in two stages: a warm-start SFT stage to produce a non-trivial starting policy, followed by on-policy distillation. Within each teacher--student pair and seed, all OPD methods share the same warm-start checkpoint, training prompts, and common OPD hyperparameters. Each SFT and OPD configuration is trained independently with 5 seeds; each resulting checkpoint is evaluated once. Method-specific supervision construction and loss safeguards are specified below. Hyperparameters for both stages are summarized in Table~\ref{tab:training_hparams}; this section describes each stage in detail.

\subsection{Warm-Start SFT}

We first SFT each student on the corresponding teacher's 10K-example corpus described in Appendix~\ref{apd:data} to obtain the warm-start checkpoint from which all OPD methods start. Sharing this initialization both gives the student a non-trivial policy before OPD~\citep{agarwal2024onpolicy_opsd} and ensures that differences across OPD methods are not confounded by SFT data or initial checkpoints.

We optimize with AdamW using a peak learning rate of $2 \times 10^{-6}$, a linear warmup ratio of $0.05$ followed by cosine decay, and train for $2$ epochs. All students share the same SFT hyperparameters (Table~\ref{tab:training_hparams}). Training is performed in bfloat16 on 8$\times$H20 96GB.

\subsection{On-Policy Distillation}

Starting from the warm-start checkpoint, we train each OPD method under an identical loop: at every step the student samples a prefix under its current policy, the teacher is queried on this prefix to produce next-token supervision, and the student is updated under the method-specific loss. Training prompts, prefix-generation settings (temperature $0.6$, top-$p=0.95$), optimizer, learning-rate schedule, batch construction, and training epochs are held identical across methods (Table~\ref{tab:training_hparams}).

For SimCT and SimpleOPD, the OPD loss is reverse KL~\citep{agarwal2024onpolicy_opsd,gu2024minillm_opd} applied in their respective supervision spaces---SimCT's space combines shared vocabulary tokens with minimal aligned units, while SimpleOPD restricts supervision to the exact teacher--student vocabulary overlap.

\paragraph{Normalization safeguard.}
\label{app:training-mask}
All main SimCT experiments use a position-level mask on the teacher normalization multiplier. For a prefix $h$, define
\begin{equation}
\label{eq:simct-mask}
Z_T(h)=\sum_{u\in\mathcal U_{\mathrm{SimCT}}(h)}\exp s_T(u\mid h),
\qquad G(h)=Z_T(h)^{-1},
\qquad m(h)=\mathbf 1\{G(h)\leq 2\}.
\end{equation}
A position contributes $m(h)\,\mathrm{KL}(q_S^{\mathrm{SimCT}}(\cdot\mid h)\|q_T^{\mathrm{SimCT}}(\cdot\mid h))$ to the distillation loss. Positions with $G(h)>2$ contribute neither loss nor gradient. We apply this mask during training for all main SimCT results. The threshold is defined on the continuation-score normalizer in Eq.~\eqref{eq:simct-projected-distribution}; original teacher support coverage is measured separately as in Appendix~\ref{app:coverage}. Appendix~\ref{app:coverage} reports the multiplier distribution and the observed masking rate.

\subsection{SimCT Algorithm}

Algorithm~\ref{alg:simct} summarizes SimCT. Given a prompt $x$, the student first generates a full on-policy response $y$. At each prediction position $t$, teacher and student next-token distributions $p_T^t$ and $p_S^t$ are evaluated under the same prefix $(x, y_{<t})$. SimCT constructs a common supervision space $\mathcal U_t$ from shared vocabulary tokens augmented with minimal aligned units induced by local tokenizer-boundary mismatches, and scoring operators $\Pi_T, \Pi_S$ use the autoregressive conditional families to produce comparable distributions $q_T^t, q_S^t$ over $\mathcal U_t$ on which reverse KL is applied.

\paragraph{Boundary alignment and long units.}
The alignment is implemented on the CPU as a greedy two-pointer scan over the teacher and student token boundaries in the same underlying text. Starting from a common boundary, the scan advances through complete tokens until the next synchronized boundary, yielding the shortest span complete under both tokenizations. Each boundary is visited at most once, so alignment is linear in the number of token boundaries. Long units are retained without length-based truncation or exclusion and receive the same average log-likelihood score as shorter units. The position-level safeguard above remains applicable independently of unit length. Appendix~\ref{app:mau-statistics} reports the empirical length distribution, including rare long units.

\input{tables/apd-code}

\input{tables/apd-hyperparameter}

%% file: tables/apd-code.tex
\begin{algorithm}[t]
\caption{Simple Cross-Tokenizer OPD (SimCT)}
\label{alg:simct}
\begin{algorithmic}[1]
\Require Teacher $\mathcal T$, student $\mathcal S_\theta$, tokenizers $\tau_T,\tau_S$,
\Statex \hspace{\algorithmicindent} vocabularies $\mathcal V_T,\mathcal V_S$, prompts $\mathcal X$, mask threshold $\gamma=2$
\Ensure Updated student $\mathcal S_\theta$
\For{each OPD training step}
    \State Sample prompts $x\sim\mathcal X$
    \State Generate student rollouts $y\sim\mathcal S_\theta(\cdot\mid x)$
    \State Tokenize the same rollout: $\mathbf y^T\gets\tau_T(y)$, $\mathbf y^S\gets\tau_S(y)$
    \State Extract $\mathcal A(y)$ by a greedy scan of synchronized token boundaries
    \State Construct $\mathcal U_{\mathrm{SimCT}}(y)\gets(\mathcal V_T\cap\mathcal V_S)\cup\mathcal A(y)$
    \State $\mathcal L_{\mathrm{SimCT}}\gets 0$
    \For{each student-generated prefix $h=(x,y_{<t})$}
        \State Set $\mathcal U_{\mathrm{SimCT}}(h)\gets\mathcal U_{\mathrm{SimCT}}(y)$
        \State Compute shared-token scores $s_M(v\mid h)=\log p_M(v\mid h)$,
        \Statex \hspace{3em} for $M\in\{T,S\}$ and $v\in\mathcal V_T\cap\mathcal V_S$
        \State Score each MAU using its model-specific token realization:
        \[
        s_M(u\mid h)=\frac{1}{k_M(u)}\sum_{j=1}^{k_M(u)}
        \log p_M(v_j\mid h,v_{<j}),\qquad M\in\{T,S\}.
        \]
        \State Compute $Z_T(h)\gets\sum_{u\in\mathcal U_{\mathrm{SimCT}}(h)}\exp s_T(u\mid h)$
        \State Set $G(h)\gets 1/Z_T(h)$
        \If{$G(h)\leq\gamma$}
            \State Normalize scores to obtain $q_T^t,q_S^t$ via Eq.~\eqref{eq:simct-projected-distribution}
            \State $\mathcal L_{\mathrm{SimCT}}\gets\mathcal L_{\mathrm{SimCT}}+\mathrm{KL}(q_S^t\|q_T^t)$
        \Else
            \State Exclude this position from the distillation loss and gradient
        \EndIf
    \EndFor
    \State Update $\theta$ using $\nabla_\theta\mathcal L_{\mathrm{SimCT}}$
\EndFor
\State \Return $\mathcal S_\theta$
\end{algorithmic}
\end{algorithm}

%% file: tables/apd-hyperparameter.tex
\begin{table}[h]
\centering
\caption{Training hyperparameters for warm-start SFT and on-policy distillation. All students share the same SFT hyperparameters; common OPD settings are shared across methods, with the SimCT-specific mask threshold listed explicitly.}
\vspace{2mm}
\label{tab:training_hparams}
\begin{tabular}{lcc}
\toprule
 & Warm-Start SFT & On-Policy Distillation \\
\midrule
Optimizer                  & AdamW                    & AdamW \\
Peak learning rate         & $2 \times 10^{-6}$       & $1 \times 10^{-6}$ \\
Warmup ratio               & $0.05$                   & $0.05$ \\
LR schedule                & Cosine decay             & Cosine decay \\
Weight decay               & $0.0$  & $0.0$ \\
Training length            & $2$ epochs               & $2$ epochs \\
Per-device batch size      & $2$              & $1$ \\
Gradient accumulation      & $4$              & $8$ \\
Effective batch size       & $64$             & $64$ \\
Max sequence length        & $4096$           & $4096$ \\
\midrule
Rollout temperature        & ---                      & $0.6$ \\
Rollout top-$p$            & ---                      & $0.95$ \\
Rollout max length         & ---                      & $4096$ \\
Loss (SimCT, SimpleOPD)    & ---                      & Reverse KL \\
Mask cutoff $G(h)$ (SimCT only) & --- & $2$ \\
\midrule
Precision                  & bfloat16               & bfloat16 \\
Hardware                   & $8\times$ H20   & $8\times$ H20 \\
\bottomrule
\end{tabular}
\end{table}

%% file: chapters/7-apd-eval.tex
\section{Evaluation Benchmarks and Metrics}
\label{apd:eval}

\subsection{Benchmarks}

We evaluate on four benchmarks covering complementary task regimes.
\textbf{GSM8K}~\citep{cobbe2021gsm8k} consists of grade-school word problems requiring multi-step arithmetic from natural-language descriptions; final answers are short numerical values, but solving the problems demands a coherent reasoning trajectory.
\textbf{MATH-500}~\citep{hendrycks2021math,lightman2023math500} draws from competition-style problems spanning algebra, geometry, number theory, and combinatorics; answers may be symbolic forms, fractions, or equations, making the task sensitive to both reasoning quality and output normalization.
\textbf{MBPP}~\citep{austin2021program_mbpp} measures function-level Python synthesis from natural-language specifications, evaluated by the dataset's provided unit tests.
\textbf{LiveCodeBench-v6}~\citep{jain2024livecodebench} evaluates code generation on contemporary problems with reduced contamination, requiring more robust algorithmic reasoning under hidden tests.

These benchmarks expose tokenizer mismatch in different ways---math evaluation involves symbolic expressions and rationale tokens, while code involves identifiers, indentation, and syntax fragments---so they jointly test whether SimCT recovers useful supervision across heterogeneous output structures rather than overfitting to a single format.

\subsection{Decoding}

All benchmarks are evaluated zero-shot: each model receives only the task prompt, with no in-context demonstrations. This isolates capability learned during SFT and OPD from prompt-level effects, and avoids tokenizer-dependent confounds introduced by few-shot examples.

For each benchmark, we use a fixed prompt template across all methods and student models. Math prompts ask the model to solve the problem and place the final answer in \verb|\boxed{}|; code prompts present the specification and ask for a Python solution in the benchmark's expected format. Full prompt templates are released in our supplementary code.

We decode with nucleus sampling at temperature $0.6$ and top-$p=0.95$, sampling one completion per instance. For each SFT or OPD configuration, the 5 runs correspond to independently trained checkpoints, each evaluated once. Unadapted Base models have no training stage and use 5 decoding seeds. The maximum generation length is fixed per benchmark and held identical across methods (4096 tokens for GSM8K, MATH-500 and LiveCodeBench-v6, and 2048 tokens for MBPP).

\subsection{Scoring}

We compute pass@1~\citep{chen2021evaluating_humaneval} separately for each run and report the mean and standard deviation of the 5 run-level scores in Table~\ref{tab:main_results_compact}. For trained models, this variability includes independent training and stochastic evaluation; it is not the standard error of the mean.

For math, we extract the final answer (preferring content inside \verb|\boxed{}| when present) and compare against the reference using \texttt{math\_verify}\footnote{\url{https://github.com/huggingface/math-verify}} under symbolic equivalence; this is the same checker used during SFT data curation (Appendix~\ref{apd:data}), ensuring the verification standard is consistent across stages. For code, generated solutions are executed in a sandboxed Python environment with a per-test timeout of 10 seconds and a memory cap of 128 MB; we use the official MBPP test harness and the LiveCodeBench-v6 official evaluator. Postprocessing rules, the answer-extraction regex, and the sandbox configuration are identical across all compared methods.

%% file: chapters/7-apd-theorem.tex
\section{Theory and Proofs}
\label{apd:theory}

This appendix provides the formal definitions and proofs for
Theorem~\ref{thm:canonical-space} and
Proposition~\ref{prop:simct-consistency}. The section has two goals. First, we
formalize why SimCT aligned units are the minimal text units that remain
unions of complete tokens under the two fixed tokenizer partitions. Second, we prove the
property used in the supervision-space ablation: merging minimal aligned units
removes exactly the within-span KL signal that distinguishes teacher and
student preferences inside each merged span.

\subsection{Tokenizer-Induced Spans and Minimal Aligned Units}
\label{apd:tokenizer-spans}

Let \(y\) be a finite text span represented over a fixed atomic string
representation, such as characters or bytes. We identify \(y\) with the interval
\([0,|y|]\). The teacher tokenizer induces a partition
\[
\mathcal P_T(y)
=
\{[b^T_{i-1}, b^T_i): i=1,\ldots,m\},
\qquad
0=b^T_0<\cdots<b^T_m=|y|,
\]
and the student tokenizer induces a partition
\[
\mathcal P_S(y)
=
\{[b^S_{j-1}, b^S_j): j=1,\ldots,n\},
\qquad
0=b^S_0<\cdots<b^S_n=|y|.
\]
Each interval in \(\mathcal P_T(y)\) is a complete teacher-token span, and
each interval in \(\mathcal P_S(y)\) is a complete student-token span. Throughout this proof, a jointly tokenizable unit means a union of complete consecutive tokens in each of these fixed partitions. The minimality statement is relative to these fixed partitions, with token boundaries inherited from the full text.

We first take the union of teacher and student token boundaries:
\[
\mathcal B_{\cup}(y)
=
\mathcal B_T(y)\cup \mathcal B_S(y),
\]
where \(\mathcal B_T(y)=\{b^T_i\}_{i=0}^m\) and
\(\mathcal B_S(y)=\{b^S_j\}_{j=0}^n\). Let
\[
0=c_0<c_1<\cdots<c_\ell=|y|
\]
be the sorted elements of \(\mathcal B_{\cup}(y)\). This gives atomic
fragments
\[
\mathcal F(y)
=
\{[c_{k-1},c_k): k=1,\ldots,\ell\}.
\]
These fragments are only used to define the construction. A fragment may lie
inside a teacher token or a student token, and therefore is not necessarily a
valid supervision unit by itself.

To obtain jointly-tokenizable units, we build an undirected graph \(G_y\) whose
vertices are the atomic fragments in \(\mathcal F(y)\). Two vertices are
connected if their fragments lie inside the same teacher-token span or inside
the same student-token span. The connected components of \(G_y\) induce a
partition of \(y\), denoted
\[
\mathcal P_{\mathrm{SimCT}}(y)
=
\{u_1,\ldots,u_r\}.
\]
We call each \(u_k\) a \emph{minimal aligned unit}. Intuitively, the connected
component closes all tokenizer overlaps: if splitting a span would cut through
a complete teacher token or a complete student token, the corresponding
fragments must remain in the same aligned unit.

\subsection{Proof of Theorem~\ref{thm:canonical-space}}
\label{apd:proof-canonical-space}

\begin{proof}
We prove joint tokenizability and minimality.

\paragraph{Joint tokenizability.}
Consider any unit \(u\in\mathcal P_{\mathrm{SimCT}}(y)\). By construction,
\(u\) is a connected component of \(G_y\). If a teacher-token span intersects
\(u\), then all atomic fragments inside that teacher token are connected in
\(G_y\), so the entire teacher-token span must lie in the same connected
component. Therefore, no boundary of \(u\) cuts through a teacher token, and
\(u\) is a union of consecutive complete teacher tokens.

The same argument applies to the student tokenizer. If a student-token span
intersects \(u\), all atomic fragments inside that student token are connected,
so the complete student-token span lies inside \(u\). Hence \(u\) is also a
union of consecutive complete student tokens. Thus \(u\) can be represented as
one or more teacher tokens and also as one or more student tokens.

\paragraph{Minimality.}
Assume for contradiction that a unit \(u\) admits a strictly finer non-empty
decomposition
\[
u = u' \circ u'',
\]
where both \(u'\) and \(u''\) are jointly tokenizable. Let \(z\) be the boundary
between \(u'\) and \(u''\). Since both sides are unions of complete teacher
tokens, \(z\) must be a teacher-token boundary. Since both sides are also
unions of complete student tokens, \(z\) must be a student-token boundary.
Therefore, no teacher token or student token crosses \(z\).

It follows that no edge in \(G_y\) can connect an atomic fragment in \(u'\) to
an atomic fragment in \(u''\), because such an edge would require the two
fragments to lie inside the same teacher-token span or the same student-token
span, contradicting that no token crosses \(z\). Thus \(u'\) and \(u''\) would
belong to different connected components of \(G_y\), contradicting that \(u\)
is one connected component. Therefore, no strictly finer non-empty
jointly-tokenizable decomposition exists.
\end{proof}

\paragraph{Implication.}
The theorem shows that SimCT does not rely on an arbitrary span length. It
selects the finest units that are valid under both tokenizers. Finer atomic
fragments may cut through complete tokens, while coarser mismatch spans merge
multiple valid aligned units.

\subsection{Proof of Proposition~\ref{prop:simct-consistency}: Coarsening and Removed KL Signal}
\label{apd:proof-coarsening}

We prove Proposition~\ref{prop:simct-consistency}, which is used in the
supervision-space ablation. For completeness, we restate the proposition before
giving the proof.
Let
\(\mathcal U_{\min}\) denote the minimal aligned-unit space for a
student-generated prefix \(h\). Let
\[
q_T^{\min}(\cdot\mid h),
\qquad
q_S^{\min}(\cdot\mid h)
\]
be the teacher and student distributions over \(\mathcal U_{\min}\). Assume \(q_T^{\min}(u\mid h)>0\) whenever \(q_S^{\min}(u\mid h)>0\), so the finite-space KL decomposition below is well defined.

A coarsening \(\mathcal C\) partitions \(\mathcal U_{\min}\) into disjoint
coarse spans. For each coarse span \(c\in\mathcal C\), the coarsened
distributions are
\[
q_T^{\mathcal C}(c\mid h)
=
\sum_{u\in c}q_T^{\min}(u\mid h),
\qquad
q_S^{\mathcal C}(c\mid h)
=
\sum_{u\in c}q_S^{\min}(u\mid h).
\]
This is exactly the aggregation used in Eq.~\ref{eq:coarsened-distribution}.

\paragraph{Proposition~\ref{prop:simct-consistency} (restated).}
For any coarsening \(\mathcal C\) of the minimal aligned-unit space,
\[
\mathrm{KL}\!\left(q_S^{\min}(\cdot\mid h)\,\|\,q_T^{\min}(\cdot\mid h)\right)
\ge
\mathrm{KL}\!\left(q_S^{\mathcal C}(\cdot\mid h)\,\|\,q_T^{\mathcal C}(\cdot\mid h)\right).
\]
Moreover, the difference equals the expected within-coarse-span KL:
\[
\begin{aligned}
&
\mathrm{KL}\!\left(q_S^{\min}(\cdot\mid h)\,\|\,q_T^{\min}(\cdot\mid h)\right)
-
\mathrm{KL}\!\left(q_S^{\mathcal C}(\cdot\mid h)\,\|\,q_T^{\mathcal C}(\cdot\mid h)\right)
\\
&\quad =
\sum_{c\in\mathcal C}
q_S^{\mathcal C}(c\mid h)
\,
\mathrm{KL}\!\left(
q_S^{\min}(\cdot\mid h,c)
\,\|\,
q_T^{\min}(\cdot\mid h,c)
\right),
\end{aligned}
\]
where
\[
q_S^{\min}(u\mid h,c)
=
\frac{q_S^{\min}(u\mid h)}{q_S^{\mathcal C}(c\mid h)},
\qquad
q_T^{\min}(u\mid h,c)
=
\frac{q_T^{\min}(u\mid h)}{q_T^{\mathcal C}(c\mid h)}
\]
for \(u\in c\), with coarse spans of zero student mass omitted and the convention $0\log 0=0$. Thus, coarsening removes exactly the KL signal corresponding to teacher--student discrepancies among minimal aligned units inside each merged span.

\begin{proof}
For readability, omit the conditioning on \(h\). Let
\(q_S=q_S^{\min}\) and \(q_T=q_T^{\min}\). For each coarse span
\(c\in\mathcal C\), write
\[
Q_S(c)=\sum_{u\in c}q_S(u),
\qquad
Q_T(c)=\sum_{u\in c}q_T(u).
\]
Then \(Q_S=q_S^{\mathcal C}\) and \(Q_T=q_T^{\mathcal C}\).

We decompose the minimal-unit KL by grouping terms according to their coarse
span:
\[
\mathrm{KL}(q_S\|q_T)
=
\sum_{c\in\mathcal C}\sum_{u\in c}
q_S(u)\log\frac{q_S(u)}{q_T(u)}.
\]
For any \(u\in c\), write
\[
q_S(u)=Q_S(c)\,q_S(u\mid c),
\qquad
q_T(u)=Q_T(c)\,q_T(u\mid c),
\]
where
\[
q_S(u\mid c)=\frac{q_S(u)}{Q_S(c)},
\qquad
q_T(u\mid c)=\frac{q_T(u)}{Q_T(c)}.
\]
Substituting this factorization gives
\[
\begin{aligned}
\mathrm{KL}(q_S\|q_T)
&=
\sum_{c\in\mathcal C}\sum_{u\in c}
Q_S(c)q_S(u\mid c)
\log
\frac{Q_S(c)q_S(u\mid c)}
{Q_T(c)q_T(u\mid c)}
\\
&=
\sum_{c\in\mathcal C}
Q_S(c)\log\frac{Q_S(c)}{Q_T(c)}
+
\sum_{c\in\mathcal C}
Q_S(c)
\sum_{u\in c}q_S(u\mid c)
\log\frac{q_S(u\mid c)}{q_T(u\mid c)}
\\
&=
\mathrm{KL}(Q_S\|Q_T)
+
\sum_{c\in\mathcal C}
Q_S(c)\,
\mathrm{KL}\!\left(q_S(\cdot\mid c)\,\|\,q_T(\cdot\mid c)\right).
\end{aligned}
\]
Since the second term is a nonnegative weighted sum of KL divergences,
\[
\mathrm{KL}(q_S\|q_T)
\ge
\mathrm{KL}(Q_S\|Q_T).
\]
Rearranging the equality above gives the stated expression for the removed KL
signal.
\end{proof}

\paragraph{Connection to the ablation.}
Equation~\ref{eq:removed-kl-signal} averages the difference above over
student-generated prefixes:
\[
\Delta_{\mathcal C}
=
\mathbb E_h
\left[
\mathrm{KL}(q_S^{\min}(\cdot\mid h)\|q_T^{\min}(\cdot\mid h))
-
\mathrm{KL}(q_S^{\mathcal C}(\cdot\mid h)\|q_T^{\mathcal C}(\cdot\mid h))
\right].
\]
By the proposition, \(\Delta_{\mathcal C}\) is nonnegative and measures the
expected within-span KL signal erased by the coarsening \(\mathcal C\). The
right panel of Figure~\ref{fig:ablation} therefore directly tests the
empirical hypothesis motivated by the theory: removing more within-span
teacher--student discrepancy may reduce downstream gains. Here the KL decomposition quantifies removed distributional signal, while the downstream relationship is evaluated empirically.

\subsection{Discussion}
\label{apd:theory-discussion}

The theorem and proposition play complementary roles. The theorem identifies
the finest jointly-tokenizable units available under the two tokenizer
partitions. The proposition shows what is lost when those units are merged:
coarsening preserves the total probability of each merged span, but discards
the conditional teacher--student differences among the minimal units inside it.

These results rely only on tokenizer boundaries and probability aggregation.
They do not assume a learned semantic alignment, an embedding similarity metric,
or an optimal-transport plan. This matches the design of SimCT: the method
changes the supervision space to the finest comparable text units while keeping
the OPD objective itself unchanged.

%% file: chapters/7-apd-computing_time.tex
\section{Computational Resources}
\label{app:compute}
\input{tables/apd-computing_time}
All experiments were conducted on a single server equipped with dual AMD EPYC 9K84 CPUs, 384 logical CPU cores, 2.2TB system memory, and $8\times$ NVIDIA H20 GPUs with approximately 96GB memory per GPU.
Table~\ref{tab:training_time} reports the measured wall-clock time of a single cross-tokenizer OPD run under this hardware setting.
Across teacher--student pairs and methods, one training run typically takes about 2.5--3 hours.
Averaged across the reported pairs, SimpleOPD takes $2.53$ hours and SimCT takes $2.69$ hours, an increase of $0.16$ hours (approximately $6.3\%$). These totals cover the full OPD training run, including student rollouts, teacher forward passes, alignment, and optimization. Boundary alignment is a linear CPU scan as described in Appendix~\ref{apd:train}. SimCT is also faster than ALM, GOLD, and DSKD across the reported settings.

Table~\ref{tab:compute} summarizes the compute required to reproduce the reported main experiments.
Each SFT and cross-tokenizer OPD configuration is trained independently with $5$ seeds. Each resulting checkpoint is evaluated once, and Table~\ref{tab:main_results_compact} aggregates the resulting run-level scores. The cold-start accounting also includes the teacher-generation runs listed in Table~\ref{tab:compute}.
For the aggregate estimate, we use a conservative rounded runtime of approximately $3$ hours per cross-tokenizer OPD run, consistent with the measured runtimes in Table~\ref{tab:training_time}.
The total cost is approximately $14.8$ days on a single $8\times$ H20 server.
The total consists of three parts.
First, the warm-start stage costs about $1.0$ day on a single $8\times$ H20 server, covering two teacher response-generation runs and three student SFT runs (Phi-4 on Qwen-generated data, Gemma-2 on Qwen-generated data, and Gemma-2 on Phi-generated data), each repeated over $5$ random seeds.
Second, cross-tokenizer OPD training costs about $9.4$ days on a single $8\times$ H20 server, covering three teacher--student pairs (Qwen2.5$\rightarrow$Phi4, Qwen2.5$\rightarrow$Gemma-2, and Phi4$\rightarrow$Gemma-2), each trained with five methods (SimpleOPD, SimCT, ALM, GOLD, and DSKD), and each method-pair configuration repeated over $5$ random seeds, yielding $3 \times 5 \times 5 = 75$ training runs.
Third, evaluation costs about $4.4$ days on a single $8\times$ H20 server, covering $3\times5$ Base-model decoding runs, $3\times5$ independently trained warm-start checkpoints, and $15\times5$ independently trained OPD checkpoints. Each trained checkpoint is evaluated once, yielding $15+15+75=105$ evaluation runs.

The reported total only accounts for the runs needed to reproduce the main experimental results.
It excludes preliminary studies, hyperparameter exploration, debugging runs, ablation studies in Section~\ref{sec:exp-supervision_recovery} and Section~\ref{sec:exp-space_ablation}, the additional analyses in Appendix~\ref{app:additional-analyses}, and failed or repeated jobs during development.
The full project compute is therefore larger than the reproducibility cost reported here.

Finally, the end-to-end wall-clock cost of cross-tokenizer OPD is dominated by student rollouts and teacher forward passes.
SimCT operates at the distillation-interface level. It introduces no trainable projection module, requires no global vocabulary-level alignment, and keeps the standard OPD training pipeline unchanged.
As a result, SimCT remains close to the lightweight SimpleOPD baseline and is faster than more complex cross-tokenizer baselines, while recovering substantially more teacher supervision than SimpleOPD.

%% file: tables/apd-computing_time.tex
\begin{table}[t]
\centering
\caption{\textbf{Training time comparison.}
We report wall-clock training time on a single machine with 8 NVIDIA H20 GPUs. All methods use the same training epochs, hardware, and common protocol; wall-clock time is measured rather than fixed.}
\vspace{2mm}
\label{tab:training_time}
\resizebox{0.9\linewidth}{!}{
\begin{tabular}{lcccc}
\toprule
\textbf{Method}
& \textbf{Qwen2.5$\rightarrow$Gemma2}
& \textbf{Phi4$\rightarrow$Gemma2}
& \textbf{Qwen2.5$\rightarrow$Phi4}
& \textbf{Average} \\
\midrule
SimpleOPD
& 2.48h & 2.61h & 2.50h & 2.53h \\
SimCT
& 2.70h & 2.73h & 2.64h & 2.69h \\
ALM
& 2.84h & 2.88h & 2.79h & 2.84h \\
GOLD
& 2.93h & 2.89h & 2.97h & 2.93h \\
DSKD
& 2.90h & 2.96h & 2.91h & 2.92h \\
\bottomrule
\end{tabular}
}
\end{table}

\begin{table}[t]
\centering
\small
\caption{
Estimated compute resources for reproducing the reported experiments. All runs use a single machine with $8\times$ NVIDIA H20 96GB GPUs. Training configurations use 5 seeds, with one evaluation per trained checkpoint; Base models use 5 decoding seeds. Cost is reported in days on the $8\times$H20 server, computed as (number of runs $\times$ wall-clock hours per run) / 24.
}
\vspace{2mm}
\label{tab:compute}
\begin{tabular}{llccc}
\toprule
Stage & Component & Runs & Runtime per run & 8$\times$H20 days \\
\midrule
\multirow{5.5}{*}{Cold start}
& Qwen2.5 response generation & $1 \times 5$ & $1$\,h & $0.21$ \\
& Phi-4 response generation & $1 \times 5$ & $1$\,h & $0.21$ \\
& Phi-4 SFT (Qwen-generated data) & $1 \times 5$ & $1$\,h & $0.21$ \\
& Gemma SFT (Qwen-generated data) & $1 \times 5$ & $1$\,h & $0.21$ \\
& Gemma SFT (Phi-generated data) & $1 \times 5$ & $1$\,h & $0.21$ \\
\cmidrule(lr){2-5}
& \textit{Cold-start subtotal} & $25$ & --- & \textit{1.04} \\
\midrule
\multirow{4.5}{*}{Training}
& Qwen $\to$ Phi, five methods & $5 \times 5$ & $3$\,h & $3.13$ \\
& Qwen $\to$ Gemma, five methods & $5 \times 5$ & $3$\,h & $3.13$ \\
& Phi $\to$ Gemma, five methods & $5 \times 5$ & $3$\,h & $3.13$ \\
\cmidrule(lr){2-5}
& \textit{Training subtotal} & $75$ & --- & \textit{9.38} \\
\midrule
\multirow{5.5}{*}{Evaluation}
& Base models (Qwen, Phi, Gemma) & $3 \times 5$ & $1$\,h & $0.63$ \\
& Warm-start (3 pairs) & $3 \times 5$ & $1$\,h & $0.63$ \\
& Qwen $\to$ Phi (5 methods) & $5 \times 5$ & $1$\,h & $1.04$ \\
& Qwen $\to$ Gemma (5 methods) & $5 \times 5$ & $1$\,h & $1.04$ \\
& Phi $\to$ Gemma (5 methods) & $5 \times 5$ & $1$\,h & $1.04$ \\
\cmidrule(lr){2-5}
& \textit{Evaluation subtotal} & $105$ & --- & \textit{4.38} \\
\midrule
\multicolumn{4}{r}{\textbf{Total}} & $\mathbf{14.79}$ \\
\bottomrule
\end{tabular}
\end{table}

%% file: chapters/7-apd-case_study.tex
\section{Case Study}
\label{app:case_study}

We provide qualitative case studies on math and code examples to illustrate how different supervision interfaces can affect cross-tokenizer OPD behavior at the local level.
For each example, we compare SimCT with representative baselines under the same teacher--student setting and report the final answer or execution result.
These cases highlight representative local errors---an incorrect modular computation, a missing symmetry factor, a mismatched input-output convention, or an invalid function signature---that are easy to miss in aggregate scores.
Several baselines recover the high-level solution template but fail on these local distinctions, which is consistent with the supervision loss caused by tokenizer mismatch as measured in Section~\ref{sec:exp-supervision_recovery}.
SimCT avoids these specific errors in the shown examples, providing qualitative intuition for how finer-grained supervision can help preserve local reasoning accuracy.
These examples complement the aggregate evaluation in Table~\ref{tab:main_results_compact}.

\newpage

\begin{figure*}[t]
\centering

\begin{tcolorbox}[
  colback=white, colframe=black!70, boxrule=0.6pt,
  title={\textbf{Case Study 1: Number Theory (Level 4) --- Qwen2.5 $\rightarrow$ Gemma-2-2B \hfill \normalfont\textit{MATH500}}},
  fonttitle=\bfseries\small,
  coltitle=white, colbacktitle=black!70,
  top=4pt, bottom=4pt, left=6pt, right=6pt
]

\noindent\colorbox{gray!10}{\parbox{0.97\linewidth}{\small
\textbf{Question:} Find the modulo $7$ remainder of the sum $1+3+5+7+9+\dots+195+197+199.$

\textbf{Ground-truth Answer:} $4$
}}

\vspace{6pt}

\renewcommand{\arraystretch}{1.4}
{\small
\begin{tabular}{@{} l p{0.62\linewidth} l @{}}
\toprule
\textbf{Method} & \textbf{Response (Abbreviated)} & \textbf{Result} \\
\midrule
\rowcolor{green!8}
SimCT (Ours) &
\textit{The sequence $1, 3, 5, \ldots, 199$ has $n = 100$ terms. Sum $= \frac{100}{2}(1 + 199) = 10000$. Since $10000 = 7 \times 1428 + 4$, the remainder is $4$.}
\newline \textbf{Answer:} $\boxed{4}$
& \textcolor{green!60!black}{\cmark\ \textbf{Correct}} \\
\midrule
SimpleOPD &
\textit{Correctly identifies 100 terms and computes the sum as 10000. However, makes an arithmetic error in the final modular division step, computing $10000 \div 7 = 1428$ remainder $3$ instead of $4$.}
\newline \textbf{Answer:} $3$
& \textcolor{red!70!black}{\xmark\ \textbf{Incorrect}} \\
\midrule
\textsc{ALM} &
\textit{Incorrectly counts the number of terms as 199 (confusing the last term with the count). Computes a wrong sum and obtains remainder $1$.}
\newline \textbf{Answer:} $1$
& \textcolor{red!70!black}{\xmark\ \textbf{Incorrect}} \\
\midrule
\textsc{DSKD} &
\textit{Correctly sets up the arithmetic sequence formula but makes an error in computing $10000 \bmod 7$, writing the valid identity $10000 = 7 \times 1429 - 3$ but reporting the incorrect nonnegative remainder $2$.}
\newline \textbf{Answer:} $2$
& \textcolor{red!70!black}{\xmark\ \textbf{Incorrect}} \\
\midrule
\textsc{GOLD} &
\textit{Identifies the correct sum of 10000 but then incorrectly applies modular arithmetic, computing $10000 \bmod 7 = 3$ (same error as SimpleOPD).}
\newline \textbf{Answer:} $3$
& \textcolor{red!70!black}{\xmark\ \textbf{Incorrect}} \\
\bottomrule
\end{tabular}
}

\vspace{8pt}

\begin{tcolorbox}[
  colback=green!4, colframe=green!50!black, boxrule=0.4pt,
  left=5pt, right=5pt, top=4pt, bottom=4pt,
  title={\small\textbf{Analysis}}, fonttitle=\bfseries\small,
  coltitle=green!50!black, colbacktitle=green!10
]
\small
SimCT correctly counts 100 terms, obtains the sum 10000, and computes $10000 = 7 \times 1428 + 4$. ALM already errs in counting the terms; SimpleOPD, DSKD, and GOLD fail at the remainder calculation. These responses illustrate distinct local reasoning errors.
\end{tcolorbox}

\end{tcolorbox}

\label{fig:case-study-1}
\end{figure*}

\clearpage

\begin{figure*}[t]
\centering

\begin{tcolorbox}[
  colback=white, colframe=black!70, boxrule=0.6pt,
  title={\textbf{Case Study 2: Counting \& Probability (Level 4) --- Qwen2.5 $\rightarrow$ Phi4-mini \hfill \normalfont\textit{MATH500}}},
  fonttitle=\bfseries\small,
  coltitle=white, colbacktitle=black!70,
  top=4pt, bottom=4pt, left=6pt, right=6pt
]

\noindent\colorbox{gray!10}{\parbox{0.97\linewidth}{\small
\textbf{Question:} In how many ways can 8 people be seated around a square table with 2 people on a side? (Two configurations are considered equivalent if one is a rotation of another.)

\textbf{Ground-truth Answer:} $10080$
}}

\vspace{6pt}

\renewcommand{\arraystretch}{1.4}
{\small
\begin{tabular}{@{} l p{0.62\linewidth} l @{}}
\toprule
\textbf{Method} & \textbf{Response (Abbreviated)} & \textbf{Result} \\
\midrule
\rowcolor{green!8}
SimCT (Ours) &
\textit{Total linear arrangements: $8! = 40320$. A square table has 4 rotational symmetries (0\textdegree, 90\textdegree, 180\textdegree, 270\textdegree). Dividing: $8!/4 = 40320/4 = 10080$.}
\newline \textbf{Answer:} $\boxed{10080}$
& \textcolor{green!60!black}{\cmark\ \textbf{Correct}} \\
\midrule
SimpleOPD &
\textit{Attempts Burnside's lemma but incorrectly computes $\mathrm{Fix}(180^\circ) = 4! \times 2^4 = 384$ instead of 0, yielding $\frac{40320 + 384}{4} = 10176$.}
\newline \textbf{Answer:} $10176$
& \textcolor{red!70!black}{\xmark\ \textbf{Incorrect}} \\
\midrule
\textsc{ALM} &
\textit{Incorrectly treats the problem as a circular permutation first ($8!/8 = 7!$), then divides by 4 again for square symmetry, double-counting: $7!/4 = 1260$.}
\newline \textbf{Answer:} $1260$
& \textcolor{red!70!black}{\xmark\ \textbf{Incorrect}} \\
\midrule
\textsc{DSKD} &
\textit{Fixes one person at a corner, then incorrectly partitions the remaining 7 people into ``pairs'' using $\frac{7!}{(2!)^3} = 630$, conflating seating with pairing.}
\newline \textbf{Answer:} $630$
& \textcolor{red!70!black}{\xmark\ \textbf{Incorrect}} \\
\midrule
\textsc{GOLD} &
\textit{Fixes one person, computes $7!$, then multiplies by $(2!)^4 = 16$ (incorrectly adding intra-side permutations already counted), yielding $\frac{7! \times 16}{4} = 20160$.}
\newline \textbf{Answer:} $20160$
& \textcolor{red!70!black}{\xmark\ \textbf{Incorrect}} \\
\bottomrule
\end{tabular}
}

\vspace{8pt}

\begin{tcolorbox}[
  colback=green!4, colframe=green!50!black, boxrule=0.4pt,
  left=5pt, right=5pt, top=4pt, bottom=4pt,
  title={\small\textbf{Analysis}}, fonttitle=\bfseries\small,
  coltitle=green!50!black, colbacktitle=green!10
]
\small
SimCT applies the clean and correct approach: total permutations divided by rotational symmetries ($8!/4$).
Each baseline attempts a different---and incorrect---strategy: SimpleOPD misapplies Burnside's lemma; \textsc{ALM} double-divides by both circular and square symmetries; \textsc{DSKD} confuses seating with pairing; and \textsc{GOLD} redundantly multiplies by intra-side arrangements.
The shown responses differ in how they handle rotational equivalence; SimCT gives the correct count in this example.
\end{tcolorbox}

\end{tcolorbox}

\label{fig:case-study-2}
\end{figure*}

\clearpage

\begin{figure*}[t]
\centering

\begin{tcolorbox}[
  colback=white, colframe=black!70, boxrule=0.6pt,
  title={\textbf{Case Study 3: Code Generation --- Qwen2.5 $\rightarrow$ Gemma-2-2B \hfill \normalfont\textit{LiveCodeBench}}},
  fonttitle=\bfseries\small,
  coltitle=white, colbacktitle=black!70,
  top=4pt, bottom=4pt, left=6pt, right=6pt
]

\noindent\colorbox{gray!10}{\parbox{0.97\linewidth}{\small
\textbf{Problem (abc389\_b):} You are given an integer $X$ not less than 2. Find the positive integer $N$ such that $N! = X$. It is guaranteed that there is exactly one such $N$.

\textbf{Constraints:} $2 \leq X \leq 3 \times 10^{18}$
}}

\vspace{6pt}

\renewcommand{\arraystretch}{1.3}
{\small
\begin{tabular}{@{} l p{0.55\linewidth} c l @{}}
\toprule
\textbf{Method} & \textbf{Generated Code (Key Logic)} & \textbf{Tests} & \textbf{Result} \\
\midrule
\rowcolor{green!8}
SimCT (Ours) &
\begin{lstlisting}[style=pythonstyle]
def findN(x):
  n = 1
  fact = 1
  while fact < x:
    n += 1
    fact *= n
  return n
x = int(input().strip())
print(findN(x))
\end{lstlisting}
& 42/42 & \textcolor{green!60!black}{\cmark\ \textbf{Pass}} \\
\midrule
SimpleOPD &
\begin{lstlisting}[style=pythonstyle]
def find_n():
  for i in range(2, int(1e18) + 1):
    if factorial(i) == int(input()):
      return i
\end{lstlisting}
\vspace{-4pt}
{\scriptsize\textcolor{red!70!black}{$\triangleright$ Calls \texttt{input()} again if the first candidate fails}}
& 3/42 & \textcolor{red!70!black}{\xmark\ \textbf{Fail}} \\
\midrule
\textsc{ALM} &
\begin{lstlisting}[style=pythonstyle]
def find_N():
  X = int(input())
  N = 1
  while True:
    N = factorial(N)  # overwrites N!
    if N == X: return N
    else: N += 1
\end{lstlisting}
\vspace{-4pt}
{\scriptsize\textcolor{red!70!black}{$\triangleright$ Overwrites the index $N$ with $N!$ and can return $X$ instead of $N$}}
& 0/42 & \textcolor{red!70!black}{\xmark\ \textbf{Fail}} \\
\midrule
\textsc{DSKD} &
\begin{lstlisting}[style=pythonstyle]
# Outputs explanation text before code
def find_n():
  X = int(input())
  for n in range(2, X + 1):
    if factorial(n) == X: return n
\end{lstlisting}
& 0/42 & \textcolor{red!70!black}{\xmark\ \textbf{Fail}} \\
\midrule
\textsc{GOLD} &
\begin{lstlisting}[style=pythonstyle]
def find_n():
  for n in range(2, 3*10**18 + 1):
    if factorial(n) == int(input()):
      return n
\end{lstlisting}
\vspace{-4pt}
{\scriptsize\textcolor{red!70!black}{$\triangleright$ Calls \texttt{input()} again if the first candidate fails}}
& 3/42 & \textcolor{red!70!black}{\xmark\ \textbf{Fail}} \\
\bottomrule
\end{tabular}
}

\vspace{8pt}

\begin{tcolorbox}[
  colback=green!4, colframe=green!50!black, boxrule=0.4pt,
  left=5pt, right=5pt, top=4pt, bottom=4pt,
  title={\small\textbf{Analysis}}, fonttitle=\bfseries\small,
  coltitle=green!50!black, colbacktitle=green!10
]
\small
SimCT reads the input once and uses $O(N)$ multiplications to reach $X=N!$; the constraints imply $N\leq20$. SimpleOPD and GOLD reread input inside the loop, whereas ALM overwrites the search index with a factorial. Test counts refer to the complete evaluated outputs.
\end{tcolorbox}

\end{tcolorbox}

\label{fig:case-study-3}
\end{figure*}

\clearpage

\begin{figure*}[t]
\centering

\begin{tcolorbox}[
  colback=white, colframe=black!70, boxrule=0.6pt,
  title={\textbf{Case Study 4: Code Generation --- Qwen2.5 $\rightarrow$ Phi4-mini \hfill \normalfont\textit{MBPP}}},
  fonttitle=\bfseries\small,
  coltitle=white, colbacktitle=black!70,
  top=4pt, bottom=4pt, left=6pt, right=6pt
]

\noindent\colorbox{gray!10}{\parbox{0.97\linewidth}{\small
\textbf{Problem:} Write a python function to find the largest number that can be formed with the given digits.

\textbf{Test Cases:}\\
\texttt{assert find\_Max\_Num([1,2,3], 3) == 321}\\
\texttt{assert find\_Max\_Num([4,5,6,1], 4) == 6541}\\
\texttt{assert find\_Max\_Num([1,2,3,9], 4) == 9321}
}}

\vspace{6pt}

\renewcommand{\arraystretch}{1.3}
{\small
\begin{tabular}{@{} l p{0.55\linewidth} l @{}}
\toprule
\textbf{Method} & \textbf{Generated Code (Key Logic)} & \textbf{Result} \\
\midrule
\rowcolor{green!8}
SimCT (Ours) &
\begin{lstlisting}[style=pythonstyle]
def find_Max_Num(digits, n):
    digits.sort(reverse=True)
    return int(''.join(map(str, digits[:n])))
\end{lstlisting}
\vspace{-4pt}
{\scriptsize\textcolor{green!60!black}{$\triangleright$ Correct signature \texttt{(digits, n)} matching test cases}}
& \textcolor{green!60!black}{\cmark\ \textbf{Pass}} \\
\midrule
SimpleOPD &
\begin{lstlisting}[style=pythonstyle]
def find_Max_Num(digits):
    sorted_digits = sorted(digits, reverse=True)
    return int(''.join(map(str, sorted_digits)))
\end{lstlisting}
\vspace{-4pt}
{\scriptsize\textcolor{red!70!black}{$\triangleright$ Missing parameter \texttt{n}; call \texttt{find\_Max\_Num([1,2,3], 3)} raises \texttt{TypeError}}}
& \textcolor{red!70!black}{\xmark\ \textbf{Fail}} \\
\midrule
\textsc{ALM} &
\begin{lstlisting}[style=pythonstyle]
def find_Max_Num(digits):
    digits.sort(reverse=True)
    return int(''.join(map(str, digits)))
\end{lstlisting}
\vspace{-4pt}
{\scriptsize\textcolor{red!70!black}{$\triangleright$ Missing parameter \texttt{n}; same \texttt{TypeError} on test invocation}}
& \textcolor{red!70!black}{\xmark\ \textbf{Fail}} \\
\midrule
\textsc{DSKD} &
\begin{lstlisting}[style=pythonstyle]
def find_Max_Num(digits):
    return int(''.join(sorted(map(str,digits),
                              reverse=True)))
\end{lstlisting}
\vspace{-4pt}
{\scriptsize\textcolor{red!70!black}{$\triangleright$ Missing parameter \texttt{n}; the two-argument test call raises \texttt{TypeError}}}
& \textcolor{red!70!black}{\xmark\ \textbf{Fail}} \\
\midrule
\textsc{GOLD} &
\begin{lstlisting}[style=pythonstyle]
def find_Max_Num(digits):
    digits.sort(reverse=True)
    return int(''.join(map(str, digits)))
\end{lstlisting}
\vspace{-4pt}
{\scriptsize\textcolor{red!70!black}{$\triangleright$ Missing parameter \texttt{n}; identical failure mode as ALM}}
& \textcolor{red!70!black}{\xmark\ \textbf{Fail}} \\
\bottomrule
\end{tabular}
}

\vspace{8pt}

\begin{tcolorbox}[
  colback=green!4, colframe=green!50!black, boxrule=0.4pt,
  left=5pt, right=5pt, top=4pt, bottom=4pt,
  title={\small\textbf{Analysis}}, fonttitle=\bfseries\small,
  coltitle=green!50!black, colbacktitle=green!10
]
\small
This example reveals a subtle but critical failure: all baseline methods generate \emph{functionally correct} sorting logic, but define the function with only one parameter \texttt{(digits)} instead of the required two-parameter signature \texttt{(digits, n)}.
When the test harness calls \texttt{find\_Max\_Num([1,2,3], 3)}, all baselines raise a \texttt{TypeError} due to the unexpected second argument.
SimCT supplies the required two-argument signature and passes the displayed tests. The baseline outputs contain function-interface errors.
\end{tcolorbox}

\end{tcolorbox}

\label{fig:case-study-4}
\end{figure*}

%% file: chapters/7-apd-additional_analyses.tex
\section{Additional Analyses of Recovered Supervision}
\label{app:additional-analyses}

This section examines which teacher predictions are represented by the candidate space, what the recovered units look like, and where their supervision helps. We distinguish original next-token support coverage from the normalized continuation scores used for optimization.

\subsection{Candidate-Space Coverage and Normalization}
\label{app:coverage}

\paragraph{Support coverage versus score normalization.}
Let $\mathcal R_T(h)\subseteq\mathcal V_T$ denote the set of original teacher token identities represented by shared tokens or by the teacher-side realizations of aligned MAUs at prefix $h$. A teacher token is counted once in this set even if it is represented more than once. Define
\begin{equation}
\label{eq:teacher-support-coverage}
M(h)=\sum_{v\in\mathcal R_T(h)}p_T(v\mid h),
\qquad r(h)=1-M(h).
\end{equation}
This diagnostic sums probabilities in the original teacher vocabulary. In particular, $\mathcal R_T(h)$ is a set of teacher tokens, whereas $\mathcal U_{\mathrm{SimCT}}(h)$ contains shared tokens and potentially multi-token continuation units. The coverage statistic $M(h)$ and the continuation-score normalizer $Z_T(h)$ in Eq.~\eqref{eq:simct-mask} therefore measure different quantities.

Table~\ref{tab:coverage-normalization} reports measurements over $263{,}817$ on-policy rollout positions. Mean support coverage is $99.5\%$, and the overall 95th percentile of $G(h)$ is $1.008$. The $1.13\%$ of positions with $G(h)>2$ are excluded from the distillation objective by the training rule in Appendix~\ref{app:training-mask}.

\begin{table}[htbp]
\centering
\small
\setlength{\tabcolsep}{3pt}
\caption{Original teacher support coverage and the actual score-normalization multiplier. $P_5$ and $P_{95}$ denote percentiles. All mass and frequency columns are percentages; $G$ is dimensionless.}
\label{tab:coverage-normalization}
\begin{tabular}{lrrrrrr}
\toprule
Pair & Positions & Mean $M$ & $P_5(M)$ & $M<.95$ & $P_{95}(G)$ & $G>2$ \\
\midrule
Qwen $\to$ Phi & 59,585 & 99.9 & 100.0 & 0.1 & 1.000 & 0.11 \\
Qwen $\to$ Gemma & 110,200 & 99.3 & 99.9 & 1.6 & 1.018 & 1.62 \\
Phi $\to$ Gemma & 94,032 & 99.4 & 99.7 & 1.6 & 1.026 & 1.20 \\
\midrule
Overall & 263,817 & 99.5 & 99.9 & 1.3 & 1.008 & 1.13 \\
\bottomrule
\end{tabular}
\end{table}

For a retained shared token $v$, the score remains $\log p_T(v\mid h)$, so
\begin{equation}
\label{eq:shared-token-rescaling}
q_T^{\mathrm{SimCT}}(v\mid h)=G(h)p_T(v\mid h),
\qquad
\frac{q_T^{\mathrm{SimCT}}(v_i\mid h)}{q_T^{\mathrm{SimCT}}(v_j\mid h)}
=\frac{p_T(v_i\mid h)}{p_T(v_j\mid h)}.
\end{equation}
Thus, relative rankings and probability ratios are preserved \emph{among shared tokens}. The reported 95th percentile places the multiplier $G(h)$ at or below $1.008$ for at least $95\%$ of measured positions, bounding upward rescaling at those positions. For an MAU, $\exp s_T(u\mid h)$ is a geometric mean of conditional token probabilities under the teacher's realization. Support coverage $M(h)$ summarizes represented mass in the original teacher vocabulary, while $q_T^{\mathrm{SimCT}}$ normalizes these model-specific continuation scores over the candidate space.

\paragraph{Recovered and residual support.}
Table~\ref{tab:shared-recovered-mass} compares exact shared-token support with the support represented after adding MAUs. The overall mean increases from $95.6\%$ to $99.5\%$. The mean residual is $0.5\%$; Tables~\ref{tab:residual-tails} and~\ref{tab:topk-outside} further characterize its distribution and the high-probability predictions it contains.

\begin{table}[htbp]
\centering
\small
\caption{Mean original teacher probability mass (\%). Residual mass is measured outside $\mathcal R_T(h)$.}
\label{tab:shared-recovered-mass}
\begin{tabular}{lrrr}
\toprule
Pair & Shared tokens & With MAUs & Residual \\
\midrule
Qwen $\to$ Phi & 99.3 & 99.9 & 0.1 \\
Qwen $\to$ Gemma & 94.8 & 99.3 & 0.7 \\
Phi $\to$ Gemma & 94.0 & 99.4 & 0.6 \\
\midrule
Overall & 95.6 & 99.5 & 0.5 \\
\bottomrule
\end{tabular}
\end{table}

\paragraph{High-probability predictions and tail cases.}
Let $\mathcal K_T(h)$ be the teacher's top-$K$ next-token set. Probability-weighted recall is
\begin{equation}
\label{eq:weighted-topk-recall}
\operatorname{Recall}@K(h)=
\frac{\sum_{v\in\mathcal K_T(h)\cap\mathcal R_T(h)}p_T(v\mid h)}
{\sum_{v\in\mathcal K_T(h)}p_T(v\mid h)}.
\end{equation}
The complementary quantity $\operatorname{MassOutside}@K(h)=1-\operatorname{Recall}@K(h)$ measures the fraction of top-$K$ probability mass outside the represented support. Table~\ref{tab:weighted-topk-recall} summarizes recall across positions. Tables~\ref{tab:residual-tails} and~\ref{tab:topk-outside} separate math and code and expose tail cases that an overall mean can hide.

\begin{table}[htbp]
\centering
\small
\caption{Probability-weighted top-$K$ recall across rollout positions (fractions, rounded to three decimal places).}
\label{tab:weighted-topk-recall}
\begin{tabular}{lrrrr}
\toprule
Statistic & $K=1$ & $K=5$ & $K=20$ & $K=100$ \\
\midrule
Mean & 0.995 & 0.995 & 0.995 & 0.995 \\
Median & 1.000 & 1.000 & 1.000 & 1.000 \\
5th percentile & 1.000 & 1.000 & 0.999 & 0.999 \\
\bottomrule
\end{tabular}
\end{table}

\begin{table}[htbp]
\centering
\small
\caption{Fraction of positions (\%) whose residual original teacher probability mass exceeds each threshold.}
\label{tab:residual-tails}
\begin{tabular}{lrrrr}
\toprule
Domain & $r(h)>5\%$ & $r(h)>10\%$ & $r(h)>20\%$ & $r(h)>50\%$ \\
\midrule
Math & 1.49 & 1.13 & 0.82 & 0.45 \\
Code & 1.06 & 0.82 & 0.63 & 0.39 \\
Overall & 1.27 & 0.97 & 0.72 & 0.42 \\
\bottomrule
\end{tabular}
\end{table}

\begin{table}[htbp]
\centering
\small
\caption{Mean top-$K$ probability mass outside the represented support, expressed as a percentage of each position's top-$K$ mass.}
\label{tab:topk-outside}
\begin{tabular}{lrrrr}
\toprule
Domain & Top-1 & Top-5 & Top-20 & Top-100 \\
\midrule
Math & 0.52 & 0.59 & 0.60 & 0.61 \\
Code & 0.42 & 0.47 & 0.47 & 0.47 \\
Overall & 0.47 & 0.53 & 0.54 & 0.54 \\
\bottomrule
\end{tabular}
\end{table}

Residual mass exceeds $50\%$ at $0.42\%$ of measured positions. The omitted support also occasionally includes the teacher's top-1 prediction. The $G(h)>2$ mask selects positions by continuation-score amplification, while residual mass and top-$K$ recall describe coverage in the original teacher vocabulary. Together, these diagnostics characterize both the typical coverage and the observed tail cases.

\clearpage
\subsection{MAU Lengths and Surface Categories}
\label{app:mau-statistics}

We extracted $30{,}215$ recovered MAUs from $900$ on-policy responses. Table~\ref{tab:mau-lengths} reports token and character lengths. Most units are short: the overall median is two tokens for either model, and four characters. Rare long spans arise when the two tokenizations do not resynchronize over structured strings. The longest observed unit is a repeated \texttt{.invalid.test} string with $447$ characters, $70$ teacher tokens, and $140$ student tokens.

\begin{table}[htbp]
\centering
\small
\setlength{\tabcolsep}{3pt}
\caption{MAU length statistics. Each length entry is mean / median / 95th percentile / maximum. Character lengths are reported separately from model-specific token counts.}
\label{tab:mau-lengths}
\begin{tabular}{lrrrr}
\toprule
Pair & MAUs & Teacher tokens & Student tokens & Characters \\
\midrule
Qwen $\to$ Phi & 4,647 & 2.16 / 2 / 3 / 18 & 1.31 / 1 / 2 / 6 & 3.97 / 3 / 10 / 26 \\
Qwen $\to$ Gemma & 11,568 & 2.24 / 2 / 5 / 70 & 3.12 / 2 / 8 / 140 & 8.83 / 6 / 24 / 447 \\
Phi $\to$ Gemma & 14,000 & 1.88 / 2 / 4 / 22 & 2.70 / 2 / 6 / 31 & 6.87 / 4 / 21 / 156 \\
\midrule
Overall & 30,215 & 2.06 / 2 / 4 / 70 & 2.65 / 2 / 6 / 140 & 7.17 / 4 / 21 / 447 \\
\bottomrule
\end{tabular}
\end{table}

Long MAUs are neither truncated nor discarded solely because of their length. As in Eq.~\eqref{eq:simct-unit-score}, they use the average log-likelihood
\[
s_M(u\mid h)=\frac{1}{k_M(u)}\sum_{j=1}^{k_M(u)}\log p_M(v_j\mid h,v_{<j})
=\log\left(\prod_{j=1}^{k_M(u)}p_M(v_j\mid h,v_{<j})\right)^{1/k_M(u)}.
\]
This avoids accumulating negative log-probability merely because a realization contains more tokens. The score remains tied to each model's tokenization through its conditional factors and token count $k_M(u)$. The independent position mask in Appendix~\ref{app:training-mask} still applies.

\paragraph{Surface forms.}
Table~\ref{tab:mau-categories} groups the recovered units by their surface forms and gives representative segmentations. Words or word-pieces, operators or punctuation, code expressions, and numbers together account for $97.1\%$ of the units. These categories describe the observed corpus rather than a semantic requirement of the MAU construction: a minimal boundary-consistent unit need not be a complete word or expression.

\begin{table}[htbp]
\centering
\small
\setlength{\tabcolsep}{4pt}
\caption{Surface categories of the $30{,}215$ recovered MAUs. Brackets delimit individual tokens in the illustrative examples; a dash means that no example is displayed. Percentages sum to $99.9\%$ because of rounding.}
\label{tab:mau-categories}
\begin{tabular}{p{.28\linewidth}rp{.22\linewidth}p{.29\linewidth}}
\toprule
Category & \% & Teacher split & Student split \\
\midrule
Word / word-piece & 29.4 & \texttt{[Let]['s]} & \texttt{[Let's]} \\
Operator / punctuation & 25.9 & \texttt{[\textbackslash)]} & \texttt{[\textbackslash][)]} \\
Code token / expression & 21.3 & \texttt{[/][week]} & \texttt{[/week]} \\
Number / numeric & 20.5 & \texttt{[2][0]} & \texttt{[20]} \\
LaTeX / math markup & 1.3 & \texttt{[\{T]} & \texttt{[\{][T]} \\
Multi-word phrase & 0.7 & --- & --- \\
Alphanumeric (var./coef.) & 0.2 & --- & --- \\
Whitespace / newline & 0.6 & --- & --- \\
\bottomrule
\end{tabular}
\end{table}

For example, \texttt{Here's} is realized as \texttt{[Here's]} by the teacher and \texttt{[Here]['][s]} by the student in one recovered unit. Another code span, \texttt{\_sequences(n}, is realized as \texttt{[\_sequences][(n]} versus \texttt{[\_][sequences][(][n]}. These examples illustrate how a common surface span can preserve supervision when its constituent tokens cannot be compared one to one.

\clearpage
\subsection{Controlled Overlap and Same-Tokenizer References}
\label{app:overlap-reference}

\paragraph{Controlled effective vocabulary overlap.}
To vary overlap while keeping the models fixed, we use VibeThinker-3B as teacher and Qwen2.5-Coder-3B-Instruct as student, which share a tokenizer. We construct nested effective shared vocabularies at $85\%$, $70\%$, and $55\%$ overlap, representing removed tokens with synthetic MAUs following the empirical MAU length distribution. Models, data, and training are held fixed. This intervention restricts the supervision interface; it does not change either pretrained tokenizer. Full-vocabulary OPD on the same model pair provides an empirical reference with an average score of $52.89$.

Let $S_{\mathrm{shared}}$, $S_{\mathrm{SimCT}}$, and $S_{\mathrm{full}}$ denote the average scores of SimpleOPD, SimCT, and full-vocabulary OPD. We report the fraction of the reference gap closed:
\begin{equation}
\label{eq:controlled-gap-recovery}
R_{\mathrm{gap}}=\frac{S_{\mathrm{SimCT}}-S_{\mathrm{shared}}}
{S_{\mathrm{full}}-S_{\mathrm{shared}}}.
\end{equation}
Table~\ref{tab:controlled-overlap} averages GSM8K, MATH500, MBPP, and LiveCodeBench. As effective overlap decreases, SimCT's absolute gain over SimpleOPD increases from $3.44$ to $4.71$ points, while $R_{\mathrm{gap}}$ decreases from $81.7\%$ to $74.5\%$. The full-vocabulary reference remains fixed across these settings. This empirical reference measures the effect of restricting the supervision interface under a fixed shared tokenizer.

\begin{table}[htbp]
\centering
\small
\setlength{\tabcolsep}{4pt}
\caption{Controlled effective overlap with a shared-tokenizer model pair. Gain is SimCT minus SimpleOPD in score points; $R_{\mathrm{gap}}$ is defined in Eq.~\eqref{eq:controlled-gap-recovery}.}
\label{tab:controlled-overlap}
\begin{tabular}{rrrrrr}
\toprule
Overlap & SimpleOPD & SimCT & Full vocabulary & Gain & $R_{\mathrm{gap}}$ \\
\midrule
85\% & 48.68 & 52.12 & 52.89 & 3.44 & 81.7\% \\
70\% & 47.53 & 51.67 & 52.89 & 4.14 & 77.2\% \\
55\% & 46.57 & 51.28 & 52.89 & 4.71 & 74.5\% \\
\bottomrule
\end{tabular}
\end{table}

\paragraph{Reference under natural tokenizer mismatch.}
We additionally compare Phi-4-mini $\to$ Gemma-2-2B with a same-tokenizer Gemma-2-9B $\to$ Gemma-2-2B full-vocabulary reference, holding the student initialization, data, and training hyperparameters fixed. We use the instruction-tuned checkpoints (Gemma-2-2B-it, Gemma-2-9B-it, and Phi-4-mini-instruct); Table~\ref{tab:natural-reference} includes the base student and both teachers for context. SimCT improves over SimpleOPD on each of the four benchmarks and by $1.09$ points on average ($32.58$ versus $31.49$).

\begin{table}[htbp]
\centering
\small
\setlength{\tabcolsep}{3pt}
\caption{Natural mismatch and a same-tokenizer reference. The two settings use Phi-4-mini and Gemma-2-9B teachers, respectively. LCB denotes LiveCodeBench.}
\label{tab:natural-reference}
\begin{tabular}{lrrrrr}
\toprule
Model / distillation setting & GSM8K & MATH500 & MBPP & LCB & Average \\
\midrule
Base Gemma-2-2B & 61.85 & 17.96 & 36.92 & 6.65 & 30.85 \\
Teacher Gemma-2-9B & 88.25 & 43.80 & 57.00 & 10.29 & 49.84 \\
Teacher Phi-4-mini & 85.75 & 54.56 & 55.36 & 7.92 & 50.90 \\
\midrule
Phi $\to$ Gemma, SimpleOPD & 62.62 & 19.00 & 37.89 & 6.43 & 31.49 \\
Phi $\to$ Gemma, SimCT & 62.92 & 21.08 & 39.16 & 7.14 & 32.58 \\
Gemma $\to$ Gemma, full vocab. & 65.04 & 22.16 & 40.52 & 7.48 & 33.80 \\
\bottomrule
\end{tabular}
\end{table}

Using the base student's score $S_{\mathrm{base}}=30.85$ as the denominator anchor gives a separate descriptive quantity,
\begin{equation}
\label{eq:natural-reference-gain}
R_{\mathrm{base}}(S)=\frac{S-S_{\mathrm{base}}}{33.80-S_{\mathrm{base}}}.
\end{equation}
It is approximately $59\%$ for SimCT and $22\%$ for SimpleOPD. Here ``base'' denotes the unadapted Gemma-2-2B student. The fractions quantify progress toward the gain of the Gemma-2-9B full-vocabulary reference; Eq.~\eqref{eq:controlled-gap-recovery} separately measures recovery in the fixed-teacher controlled-overlap experiment.

\paragraph{Gains relative to the OPD starting point.}
For the main experiments in Table~\ref{tab:main_results_compact}, comparison with the actual SFT warm-start gives average-score gains of $1.86$ points ($30.94\to32.80$, approximately $6.0\%$) for Qwen $\to$ Gemma, $3.76$ points ($51.30\to55.06$, $7.3\%$) for Qwen $\to$ Phi, and $1.47$ points ($31.11\to32.58$, $4.7\%$) for Phi $\to$ Gemma. These improvements use the SFT warm-start as their baseline.

\subsection{Within-Unit Scoring Normalization}
\label{app:normalization-ablation}

On Qwen2.5 $\to$ Gemma-2, we compare three within-MAU scoring rules under otherwise identical training. Write the unnormalized continuation log-likelihood as
\[
\ell_M(u\mid h)=\sum_{j=1}^{k_M(u)}\log p_M(v_j\mid h,v_{<j}).
\]
The alternatives are no length normalization, $s_M=\ell_M$; token-length normalization, $s_M=\ell_M/k_M(u)$; and byte-length normalization, $s_M=\ell_M/b(u)$, where $b(u)$ is the byte length of the unit's surface text. Each variant subsequently normalizes candidate scores as in Eq.~\eqref{eq:simct-projected-distribution}.

\begin{table}[htbp]
\centering
\small
\setlength{\tabcolsep}{4pt}
\caption{Within-unit normalization ablation on Qwen2.5 $\to$ Gemma-2. Scores are reported to one decimal place; LCB denotes LiveCodeBench.}
\label{tab:normalization-ablation}
\begin{tabular}{lrrrrr}
\toprule
Scoring rule & GSM8K & MATH500 & MBPP & LCB & Average \\
\midrule
No normalization & 62.3 & 21.6 & 37.4 & 6.5 & 32.0 \\
Token-length (SimCT) & \textbf{63.5} & \textbf{22.2} & \textbf{38.2} & \textbf{7.3} & \textbf{32.8} \\
Byte-length & 62.9 & 21.5 & 37.6 & 6.9 & 32.2 \\
\bottomrule
\end{tabular}
\end{table}

Token-length normalization gives the highest reported score on each benchmark, with an average gain of $0.8$ points over no normalization and $0.6$ points over byte-length normalization. Averaging log-probabilities avoids a direct penalty from the number of factors in a unit's realization, while preserving its model-specific token factorization. These results support the practical choice in Eq.~\eqref{eq:simct-unit-score} for this model pair.

\subsection{Position-Level Supervision Analysis}
\label{app:position-analysis}

To localize the benefit of recovered supervision, we perform a controlled ablation on Qwen2.5 $\to$ Gemma-2. Each variant supervises only the indicated MAU subset. The budget is the fraction of available MAUs used in this ablation; compared subsets are matched by count. Random subsets provide controls at the same budget. We distinguish selection by teacher next-token entropy, by shared-vocabulary probability mass, and by surface content (numbers, code expressions, and operators versus word-pieces or formatting).

Let $S_{\mathrm{variant}}$ be the score of a selection variant and $S_{\mathrm{Full\ SimCT}}$ the score when all MAUs are used. The reported recovery is
\begin{equation}
\label{eq:position-recovery}
R_{\mathrm{subset}}=\frac{S_{\mathrm{variant}}-S_{\mathrm{SimpleOPD}}}
{S_{\mathrm{Full\ SimCT}}-S_{\mathrm{SimpleOPD}}}.
\end{equation}
Pure SimpleOPD defines $0\%$ recovery and Full SimCT defines $100\%$. The no-MAU anchor includes the out-of-shared-vocabulary (OoSV) supervision used in Section~\ref{sec:exp-supervision_recovery}; it is therefore distinct from pure SimpleOPD and achieves $18\%$ recovery.

\begin{table}[htbp]
\centering
\small
\setlength{\tabcolsep}{4pt}
\caption{Position-level ablation with matched MAU budgets. Recovery is $R_{\mathrm{subset}}$ from Eq.~\eqref{eq:position-recovery}. The $37\%$ budget is the matched selection ratio in this experiment.}
\label{tab:position-recovery}
\begin{tabular}{p{.33\linewidth}p{.37\linewidth}rr}
\toprule
Variant & MAU positions supervised & Budget & Recovery \\
\midrule
No-MAU (SimpleOPD + OoSV) & None & 0\% & 18\% \\
Full SimCT & All & 100\% & 100\% \\
\midrule
Random-50 & Random & 50\% & 75\% \\
High-Entropy-50 & Highest teacher entropy & 50\% & 76\% \\
Low-Entropy-50 & Lowest teacher entropy & 50\% & 74\% \\
Low-SharedMass-50 & Lowest shared-vocabulary mass & 50\% & 88\% \\
High-SharedMass-50 & Highest shared-vocabulary mass & 50\% & 56\% \\
\midrule
Content-Matched & Numbers / code / operators & 37\% & 83\% \\
Random-Matched & Random, count-matched & 37\% & 66\% \\
Other-Matched & Word-pieces / formatting & 37\% & 45\% \\
\bottomrule
\end{tabular}
\end{table}

High- and low-entropy selection give similar recovery ($76\%$ and $74\%$), close to the $75\%$ random control. Mean teacher entropy is also similar at MAU and aligned positions ($0.21$ and $0.23$ nats). These measurements do not indicate that the benefit is concentrated at high-entropy positions. In contrast, low-shared-mass selection outperforms both high-shared-mass selection and the matched random control ($88\%$ versus $56\%$ and $75\%$). Content selection also outperforms other surfaces and its random control ($83\%$ versus $45\%$ and $66\%$). Within this experiment, recovered supervision is especially useful where shared-token supervision is weak and for arithmetic or code-related surfaces. 

\subsection{Relation to Prior Methods and Off-Policy Applicability}
\label{app:method-comparison}

\paragraph{Scope of the comparison.}
Cross-tokenizer methods differ in both their alignment mechanisms and the objects they compare. Table~\ref{tab:method-interface-comparison} summarizes the relevant distinctions. 

\begin{table}[htbp]
\centering
\small
\setlength{\tabcolsep}{4pt}
\caption{Comparison of the supervision or scoring interfaces. Entries summarize alignment mechanisms and scoring objects.}
\label{tab:method-interface-comparison}
\begin{tabular}{p{.20\linewidth}p{.71\linewidth}}
\toprule
Method & Alignment and comparison \\
\midrule
GOLD~\citep{patino2025gold} & Aligns token sequences through probability merging; its hybrid vocabulary loss combines direct shared-token mapping with a sorted-probability loss for unmatched tokens. \\[3pt]
ALM~\citep{minixhofer2025alm} & Compares likelihoods of aligned chunks through binarized divergences over a chunk and its complement; it also considers hidden-state distillation. \\[3pt]
Cross-tokenizer likelihood scoring~\citep{phan2026ctls} & Uses BPE structure to evaluate a model's likelihoods under another tokenizer; provides exact special cases, a general lossless procedure, and faster approximations. \\[3pt]
SimCT & Constructs a shared-token-plus-MAU candidate space, assigns normalized continuation scores, and applies reverse KL without a trainable projection module. \\
\bottomrule
\end{tabular}
\end{table}

SimCT's theoretical minimality concerns synchronized boundaries of the two fixed tokenizations (Appendix~\ref{apd:theory}). The coarsening ablation compares alternative granularities within this construction. SimCT uses normalized continuation scores, while the likelihood-scoring method targets cross-tokenizer likelihood evaluation. Runtime comparisons in Appendix~\ref{app:compute} apply to the evaluated implementations and settings.

\paragraph{Applicability beyond on-policy prefixes.}
Constructing aligned units and scoring their continuations does not intrinsically require prefixes to be sampled from the current student. The same interface can in principle be evaluated on fixed-corpus or teacher-generated prefixes when the teacher's conditional probabilities are accessible. This changes the prefix distribution while retaining the candidate construction and distillation divergence. The training experiments in this paper evaluate the interface using student-generated on-policy prefixes.

\clearpage
\subsection{Broader Connections and Possible Extensions}
\label{app:broader-connections}

The following connections suggest extensions in supervision design, interactive training, and evaluation.

\paragraph{Input structure and supervision quality.}
SepSeq studies long numerical sequences and uses separator tokens to organize attention without training~\citep{sun2026sepseq}. This is related to our numerical and structured-text examples in the sense that the representation of local units can affect model behavior, but it changes the input rather than constructing a cross-tokenizer distillation space. Robust Preference Optimization via Dynamic Target Margins adapts preference-learning margins to improve robustness to noisy preference data~\citep{sun2025gamma}. It addresses the quality and strength of supervision; SimCT addresses whether teacher and student scores refer to comparable units. Together, these approaches motivate studying adaptive weighting over aligned supervision units.

\paragraph{Agent training data and interaction.}
LIMI studies learning agent behavior from a small set of curated demonstrations~\citep{sun2025limi}; daVinci-Dev studies agent-native trajectories grounded in real tool interactions~\citep{sun2026davincidev}; and daVinci-Env provides executable environments for software-engineering agent training~\citep{sun2026davincienv}. These works motivate examining cross-tokenizer supervision on tool calls, observations, and code edits, where surface boundaries may differ between models. ARGO uses asynchronous human guidance and action-level supervision control for long-horizon research-agent training~\citep{sun2026argo}. Its intervention and masking mechanisms address trajectory quality rather than tokenizer compatibility. These interaction structures suggest aligned-unit supervision over tool calls, observations, and code edits.

\paragraph{Evaluation beyond single responses.}
AgencyBench evaluates broad autonomous-agent capabilities in long contexts~\citep{sun2026agencybench}; LongCLI-Bench evaluates extended command-line programming workflows~\citep{sun2026longcli}; and InnovatorBench evaluates agents conducting innovative AI research~\citep{sun2026innovatorbench}. These benchmarks offer settings for evaluating whether local cross-tokenizer supervision supports long-horizon planning and execution through environment-aware rollouts and end-to-end task metrics.